\documentclass{article} %
\usepackage{iclr2027_conference,times}

\usepackage{amsmath,amsfonts,bm}

\def\eqref#1{equation~\ref{#1}}

\def\1{\bm{1}}

\DeclareMathAlphabet{\mathsfit}{\encodingdefault}{\sfdefault}{m}{sl}
\SetMathAlphabet{\mathsfit}{bold}{\encodingdefault}{\sfdefault}{bx}{n}

\DeclareMathOperator*{\argmin}{arg\,min}

\usepackage{hyperref}
\usepackage{url}

\usepackage{graphicx}
\usepackage{booktabs}
\usepackage{tabularx}
\usepackage{amsmath}
\usepackage{amssymb}
\usepackage{amsthm}
\usepackage{mathtools}
\usepackage{wrapfig}

\usepackage{multicol}
\usepackage{multirow}

\usepackage[dvipsnames]{xcolor}

\usepackage[T1]{fontenc}

\usepackage[capitalize,noabbrev,nameinlink]{cleveref}

\theoremstyle{plain}
\newtheorem{theorem}{Theorem}[section]
\newtheorem{proposition}[theorem]{Proposition}

\newtheorem{lemma}[theorem]{Lemma}
\theoremstyle{definition}

\theoremstyle{remark}

\definecolor{BrickRed}{rgb}{0.6,0,0}
\definecolor{RoyalBlue}{rgb}{0,0,0.8}
\definecolor{Tdgreen}{rgb}{0,0.4,0.7}
\definecolor{cadmiumgreen}{rgb}{0.0, 0.42, 0.24}

\hypersetup{colorlinks, citecolor=cadmiumgreen, linkcolor=cadmiumgreen}
\usepackage{letterspace}

\title{%
  \textls[-20]{Groups Don't Move as One: A Bregman}\\
  \textls[-20]{Clustering View of Linearized Group Influence}%
}
\usepackage{calc}
\DeclareRobustCommand{\topast}{\raisebox{\heightof{$\scriptstyle 1$}-\heightof{$\scriptstyle\ast$}}{$\scriptstyle\ast$}}
\renewcommand{\fnsymbol}[1]{\ifcase\value{#1}\or\topast\or\dagger\or\ddagger\fi}

\author{%
    Hyunho Choi$^{1}$, Jaeseung Heo$^{2}$, Youngbin Choi$^{2}$, MoonJeong Park$^{2}$, Dongwoo Kim$^{1,2,}$\thanks{Corresponding author.}\\
    $^{1}$Department of Computer Science \& Engineering, POSTECH\\
    $^{2}$Graduate School of Artificial Intelligence, POSTECH\\
    {\small\texttt{\{hyunhochoi,jsheo12304,choi.youngbin,mjeongp,dongwoo.kim\}@postech.ac.kr}}
}

\iclrfinalcopy %
\begin{document}

\maketitle

\begin{abstract}
Influence functions provide an efficient alternative to retraining for assessing the effects of removing or reweighting training data.
For groups, however, the standard first-order estimate reduces individual contributions to a sum that can conceal substantial differences among members, including opposing effects that cancel.
We establish a Bregman centroid formulation of the linearized group influence objective, interpreting this aggregate as a single representative and using Bregman information (BI) to quantify the variation it omits.
Extending this formulation to multiple clusters reveals distinct contributions whose sum recovers the original estimate.
We further propose an estimator that jointly optimizes the weights of these cluster contributions, rather than simply summing them.
Experiments uncover cluster-specific influence patterns and a positive association between target-specific BI and approximation error.
The proposed estimator achieves higher $R^2$ than standard group influence in predicting target changes under the nonlinear proximal Bregman response function.
Simplified clustering and correction variants reduce computational cost.
The centroid view thus explains what additive group influence leaves out and provides a basis for using that information to improve estimation.
\end{abstract}  

\section{Introduction}
Influence functions estimate how removing or reweighting training data affects a trained model without retraining it for every perturbation~\citep{jaeckel1972infinitesimal,hampel1974influence,cook1977detection,koh2017understanding}.
By approximating changes in model parameters and quantities such as validation loss, they help us understand how training data contribute to model behavior.
Group influence extends this analysis from individual examples to subsets of training data, estimating the joint effect of removing or reweighting their members~\citep{basu2020second,koh2019accuracy,heo2026interactionawareinfluencefunctionsgroup}.

Standard first-order group influence estimates this joint effect by summing individual influence estimates~\citep{koh2019accuracy}.
However, the sum hides differences among group members~\citep{basu2020second,heo2026interactionawareinfluencefunctionsgroup}.
For example, a near-zero estimated change in validation loss may reflect uniformly weak individual contributions or large opposing contributions that cancel.
We seek to characterize this hidden structure and investigate whether accounting for it can improve group influence estimation.

We show that the linearized objective underlying group influence is equivalent to a Bregman centroid problem~\citep{banerjee2005clustering,nielsen2007centroid}, providing an alternative perspective grounded in Bregman geometry.
In this formulation, the group estimate is a single representative of scaled individual parameter-change estimates.
This connection identifies both what aggregation retains, their center, and what it hides, their variation around that center.
Bregman information (BI) provides a measure of this hidden variation.

We extend this view from one centroid to several through Bregman clustering~\citep{banerjee2005clustering}.
The resulting clusters reveal how different parts of the group reinforce or oppose one another, while their contributions still sum to the original group estimate.
Building on this decomposition, we propose an estimator that assigns a separate weight to each cluster's contribution.
We jointly optimize these weights using a refined approximation of how the group affects the model.
This keeps the centroids fixed while allowing the relative contributions of different clusters to change.

Our experiments reveal distinct cluster-wise influence patterns hidden by the aggregate and show a positive association between target-specific BI and approximation error.
The proposed estimator achieves higher $R^2$ than standard group influence when evaluated against target changes under the nonlinear proximal Bregman response function (PBRF)~\citep{bae2022if}.
We also examine simpler clustering and correction schemes that reduce computation.
Together, these findings show that examining within-group structure can support both a finer understanding of group influence and more accurate estimation.

\section{Preliminaries}
\label{sec:prelim}

We introduce additive group influence and the proximal Bregman response function it approximates.
Related work is discussed in~\cref{app:related-work}.

\subsection{Influence Functions and Additive Group Influence}
\label{sec:prelim-if}

Let $\mathcal D_{\mathrm{train}}=\{(x_i,y_i)\}_{i=1}^{N}$ be the training set and $h_\theta$ a model with parameters $\theta\in\mathbb R^d$.
The per-example loss $\ell(u,y)$ is twice differentiable and convex in the model output $u$.
We write $L_i(\theta)\coloneq\ell(h_\theta(x_i),y_i)$ and $\mathcal J(\theta)\coloneq N^{-1}\sum_{i=1}^{N}L_i(\theta)$.
The \emph{original model}, trained before any data removal or reweighting, has parameters $\theta_0$.
All gradients $z_i\coloneq\nabla_\theta L_i(\theta_0)$ and curvature matrices below are evaluated at these parameters.

For a nonempty group $\mathcal S\subseteq[N]$, uniform reweighting gives
\[
    \theta^\star_{\mathcal S}(\epsilon)
    \coloneq\argmin_\theta
    \left\{
        \mathcal J(\theta)
        +\epsilon\sum_{i\in\mathcal S}L_i(\theta)
    \right\}.
\]
We call $\epsilon\sum_{i\in\mathcal S}L_i(\theta)$ the \emph{group-reweighting term}.
It changes the weights of the group's existing training losses rather than introducing a separate loss function.
We refer to the resulting change in model parameters,
$\theta^\star_{\mathcal S}(\epsilon)-\theta_0$,
as the model's \emph{parameter response} to this group reweighting.
Setting $\epsilon=-1/N$ removes the group's loss terms,
with $|\mathcal S|=1$ recovering leave-one-out removal.

Under the classical regularity conditions, influence functions
approximate this parameter response to first order in
$\epsilon$~\citep{hampel1974influence,cook1982residuals,koh2017understanding}:
\begin{equation}
    \theta^\star_{\mathcal S}(\epsilon)-\theta_0
    \;\approx\;
    -\epsilon H^{-1}\sum_{i\in\mathcal S}z_i,
    \qquad
    H\coloneq\nabla_\theta^2\mathcal J(\theta_0).
    \label{eq:classical-if}
\end{equation}
The derivation and its regularity conditions are given
in~\cref{app:classical-group-influence}.

For neural networks, $H$ can be singular or indefinite~\citep{bae2022if}.
We focus on the estimate obtained by replacing it with the damped generalized Gauss-Newton (GGN) matrix~\citep{schraudolph2002fast,martens2010deep,botev2017practical}
\[
    G\coloneq\frac{1}{N}\sum_{v=1}^{N}J_v^\top H_vJ_v+\lambda I_d,
\]
where $J_v$ is the Jacobian of $h_\theta(x_v)$ and $H_v$ is the Hessian of $\ell(\cdot,y_v)$ at $h_{\theta_0}(x_v)$.
Convexity of the loss in its output and $\lambda>0$ ensure $G\succ0$~\citep{martens2010deep}.
The resulting group estimate, denoted by $\widehat\theta_{\mathcal S}(\epsilon)$, is additive because the same inverse-curvature matrix acts on every member gradient~\citep{koh2019accuracy}:
\begin{equation}
\begin{aligned}
    \widehat\theta_{\mathcal S}(\epsilon)-\theta_0
    &=-\epsilon G^{-1}\sum_{i\in\mathcal S}z_i\\
    &=\sum_{i\in\mathcal S}
      \bigl(\widehat\theta_{\{i\}}(\epsilon)-\theta_0\bigr),
    \qquad
    \widehat\theta_{\{i\}}(\epsilon)=\theta_0-\epsilon G^{-1}z_i.
\end{aligned}
    \label{eq:additive-group-response}
\end{equation}
For a differentiable scalar target $f$, we seek to estimate
how its value changes under training-group reweighting.
For example, a held-out example $(x',y')$, which we call a
\emph{query}, defines the target
$f(\theta)=\ell(h_\theta(x'),y')$~\citep{koh2017understanding,heo2026interactionawareinfluencefunctionsgroup}.
Our formulation also allows aggregate targets, such as the
mean loss on a held-out set.
Standard group influence (GI) uses the estimated parameter
response to predict the corresponding target change:
\begin{equation}
    \widehat A_{\mathcal S}(\epsilon)
    \coloneq\nabla_\theta f(\theta_0)^\top
    \bigl(\widehat\theta_{\mathcal S}(\epsilon)-\theta_0\bigr)
    =-\epsilon\sum_{i\in\mathcal S}q^\top z_i,
    \qquad
    q\coloneq G^{-1}\nabla_\theta f(\theta_0).
    \label{eq:standard-group-influence}
\end{equation}
We refer to $q$ as the \emph{target direction} associated with $f$.

\subsection{PBRF and Its Linearized Surrogate}
\label{sec:prelim-pbrf}

\citet{bae2022if} show that damped-GGN influence estimates
approximate the parameter response defined by the proximal
Bregman response function (PBRF) to first order in the
reweighting strength.
Writing this construction for a reweighted group
$\mathcal S$ gives the objective
\begin{equation}
    Q^{\mathrm{PBRF}}_{\mathcal S}(\theta;\epsilon)
    \coloneq
    \frac{1}{N}\sum_{v=1}^{N}
    D_{\ell(\cdot,y_v)}
    \bigl(h_\theta(x_v),h_{\theta_0}(x_v)\bigr)
    +\epsilon\sum_{i\in\mathcal S}L_i(\theta)
    +\frac{\lambda}{2}\|\theta-\theta_0\|_2^2,
    \label{eq:pbrf}
\end{equation}
where
$D_\phi(u,v)\coloneq
\phi(u)-\phi(v)-\nabla\phi(v)^\top(u-v)$
is the Bregman divergence generated by a differentiable
convex function $\phi$~\citep{bregman1967relaxation, banerjee2005clustering}.
The first and third terms penalize changes in model outputs
and parameters, respectively, while the second is the
group-reweighting term introduced in~\cref{sec:prelim-if}.
Following this characterization, we use the nonlinear PBRF
response as our evaluation reference, rather than the
response to the reweighted training objective.

Since $\theta_0$ minimizes the PBRF objective at $\epsilon=0$, we use a local surrogate around $\theta_0$.
Dropping the parameter-independent constant
$\epsilon\sum_{i\in\mathcal S}L_i(\theta_0)$ gives the
\emph{linearized group PBRF objective}:
\begin{equation}
    Q^{\mathrm{lin}}_{\mathcal S}(\theta;\epsilon)
    \coloneq D_{\Phi_G}(\theta,\theta_0)
    +\epsilon c_{\mathcal S}^\top(\theta-\theta_0),
    \qquad
    \Phi_G(\theta)\coloneq\tfrac12\theta^\top G\theta,
    \quad
    c_{\mathcal S}\coloneq\sum_{i\in\mathcal S}z_i.
    \label{eq:pbrf-lin}
\end{equation}
Here $D_{\Phi_G}(\theta,\theta')=\tfrac12(\theta-\theta')^\top G(\theta-\theta')$.
Since $G\succ0$, the unique minimizer is $\widehat\theta_{\mathcal S}(\epsilon)=\theta_0-\epsilon G^{-1}c_{\mathcal S}$, recovering~\cref{eq:additive-group-response}.
The second-order expansion of the first term,
together with the third term, has curvature $G$ at
$\theta_0$.
Thus, $D_{\Phi_G}$ expresses this local penalty in parameter
space, whereas $D_{\ell(\cdot,y_v)}$ in~\cref{eq:pbrf}
acts in model-output space. The derivation and the local regularity conditions relating
this estimate to the first-order PBRF response are given
in~\cref{app:pbrf-linearization}.

\begin{figure}[t]
  \centering
  \includegraphics[width=\linewidth]{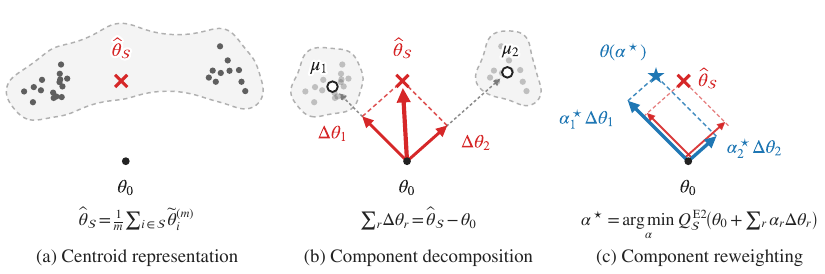}
  \caption{
    Illustration of the Bregman clustering view of linearized group influence.
    \textbf{(a)} Standard GI summarizes member-wise response points by a single centroid, hiding within-group structure.
    \textbf{(b)} Bregman clustering reveals distinct component contributions while preserving the original aggregate estimate.
    \textbf{(c)} Joint E2 reweights these contributions under a corrected objective, yielding a refined estimate (blue star).
    }
  \label{fig:illustration}
\end{figure}

\section{A Bregman Clustering View of Group Influence}
\label{sec:method}

The additive estimate in~\cref{eq:additive-group-response} describes the group's overall response but not the variation among its members, as illustrated in~\cref{fig:illustration}(a).
We characterize this aggregation by showing that the linearized group PBRF objective is a Bregman centroid objective.
This identifies both the single representative retained by GI and the dispersion it omits, and leads to a multi-component representation that preserves the original estimate, as illustrated in~\cref{fig:illustration}(b).

\subsection{Group Influence as a Bregman Centroid}
\label{sec:method-point-cloud}
\label{sec:method-clustering}

Let $m\coloneq|\mathcal S|$.
To expose the member-wise structure of~\cref{eq:pbrf-lin}, first write it as an average:
\begin{equation}
    Q^{\mathrm{lin}}_{\mathcal S}(\theta;\epsilon)
    =\frac{1}{m}\sum_{i\in\mathcal S}
    \left[
        D_{\Phi_G}(\theta,\theta_0)
        +\epsilon m z_i^\top(\theta-\theta_0)
    \right].
    \label{eq:distributed-group-objective}
\end{equation}
Each summand has a unique minimizer,
\begin{equation}
\begin{aligned}
    \widetilde\theta_i^{(m)}(\epsilon)
    &\coloneq\argmin_{\vartheta\in\mathbb R^d}
    \left\{
        D_{\Phi_G}(\vartheta,\theta_0)
        +\epsilon m z_i^\top(\vartheta-\theta_0)
    \right\}\\
    &=\theta_0-\epsilon mG^{-1}z_i.
\end{aligned}
    \label{eq:individual-response-centers}
\end{equation}
We call these minimizers \emph{member-wise response points} and collect them in the response cloud
$\mathcal P_{\mathcal S}(\epsilon)\coloneq\{\widetilde\theta_i^{(m)}(\epsilon):i\in\mathcal S\}$.
Relative to the singleton estimate in~\cref{eq:additive-group-response}, each point scales the estimated parameter displacement by $m$.

A Bregman centroid minimizes the average divergence from a collection of points to one representative~\citep{banerjee2005clustering}.
The next theorem identifies the entire linearized group objective with such a centroid objective.

\begin{theorem}[Bregman centroid representation of the linearized group PBRF]
\label{thm:group-bregman-clustering}
Let $G\succ0$, let $\Phi_G$ and $Q^{\mathrm{lin}}_{\mathcal S}$ be as in~\cref{eq:pbrf-lin}, and let $\mathcal S\subseteq[N]$ have size $m$ with response points defined by~\cref{eq:individual-response-centers}.
For every $\theta\in\mathbb R^d$,
\begin{align}
    Q^{\mathrm{lin}}_{\mathcal S}(\theta;\epsilon)
    =\frac{1}{m}\sum_{i\in\mathcal S}
    D_{\Phi_G}\bigl(\widetilde\theta_i^{(m)}(\epsilon),\theta\bigr)
    +K_{\mathcal S}(\epsilon),
    \label{eq:group-clustering-objective}
\end{align}
where $K_{\mathcal S}(\epsilon)\coloneq-\frac{\epsilon^2m}{2}\sum_{i\in\mathcal S}z_i^\top G^{-1}z_i$ is independent of $\theta$.
The unique minimizer is
\begin{equation}
    \widehat\theta_{\mathcal S}(\epsilon)
    \;=\;\frac{1}{m}\sum_{i\in\mathcal S}\widetilde\theta_i^{(m)}(\epsilon)
    \;=\;\theta_0-\epsilon G^{-1}c_{\mathcal S}.
    \label{eq:response-point-centroid-preview}
\end{equation}
\end{theorem}

The proof is given in~\cref{app:proof-group-clustering}.
\Cref{app:center-shift} gives a dual-coordinate extension to general Legendre generators and states where the quadratic specialization is needed.
The result is an objective identity for every $\theta$, not merely an equality of minimizers.
It shows that standard GI fits a single representative to the response cloud in the geometry induced by $G$.
The centroid objective also provides a way to quantify the variation hidden by this representative.

\subsection{Bregman Information of the Response Cloud}
\label{sec:method-information}

Groups with the same centroid can have different member-wise responses.
For a scalar target, for example, a small aggregate can arise from
uniformly weak contributions or from large opposing contributions
that cancel.
\emph{Bregman information} (BI) measures this omitted dispersion
as the minimum average divergence from the response points to
a single representative~\citep{banerjee2005clustering},
corresponding to the minimum value of the first term on the
right-hand side of~\cref{eq:group-clustering-objective}.

We first measure dispersion in parameter space, without
reference to a particular target.
The minimum is attained at the group centroid
$\widehat\theta_{\mathcal S}(\epsilon)$.
Because each member-wise displacement
in~\cref{eq:individual-response-centers} is scaled by
$\epsilon m$, the resulting quadratic divergence contains
a common factor $\epsilon^2m^2$.
We separate this explicit scaling from the within-group
dispersion by writing
\[
    \frac{1}{m}\sum_{i\in\mathcal S}
    D_{\Phi_G}\bigl(
        \widetilde\theta_i^{(m)}(\epsilon),
        \widehat\theta_{\mathcal S}(\epsilon)
    \bigr)
    =\epsilon^2m^2\mathcal B_G(\mathcal S),
\]
where the \emph{normalized full Bregman information} is
\begin{equation}
    \mathcal B_G(\mathcal S)
    \coloneq\frac{1}{2m}\sum_{i\in\mathcal S}
    (z_i-\bar z_{\mathcal S})^\top
    G^{-1}(z_i-\bar z_{\mathcal S}),
    \qquad
    \bar z_{\mathcal S}
    \coloneq\frac{1}{m}\sum_{i\in\mathcal S}z_i.
    \label{eq:full-bregman-information}
\end{equation}

To measure how each response differs in their effects
on a scalar target $f$, we instead consider their
first-order target changes:
\begin{equation}
    f\bigl(\widetilde\theta_i^{(m)}(\epsilon)\bigr)-f(\theta_0)
    \;\approx \; \nabla_\theta f(\theta_0)^\top
    \bigl(\widetilde\theta_i^{(m)}(\epsilon)-\theta_0\bigr)
    \;=\;-\epsilon m q^\top z_i,
    \label{eq:target-projection}
\end{equation}
where $q=G^{-1}\nabla_\theta f(\theta_0)$ is the target
direction defined in~\cref{eq:standard-group-influence}.
The mean of these linearized target changes recovers
standard GI, while their dispersion describes variation
hidden by this mean.
We measure dispersion among these linearized target changes
using half the squared difference,
$D_\psi(a,b)=\tfrac12(a-b)^2$,
which is the Bregman divergence generated by $\psi(a)=a^2/2$.
The optimal single representative under this divergence is
the mean, and the minimum average divergence is therefore
half the variance of the linearized target changes.
Removing the common factor $\epsilon^2m^2$ gives the
\emph{normalized projected Bregman information}:
\begin{equation}
    \mathcal B_{\mathrm{proj}}(\mathcal S)
    \coloneq\frac{1}{2m}\sum_{i\in\mathcal S}
    \bigl[q^\top(z_i-\bar z_{\mathcal S})\bigr]^2
    =\frac12\operatorname{Var}_{i\sim\mathrm{Unif}(\mathcal S)}
    (q^\top z_i).
    \label{eq:projected-bregman-information}
\end{equation}
Projected BI therefore measures variation among the
unscaled influence scores for the chosen target, rather
than variation across the full parameter-response space.
The normalization identities are proved
in~\cref{app:information-identities}.

Both quantities characterize dispersion omitted by the
centroid; neither by itself bounds the error relative to
the nonlinear PBRF.
We examine the association between projected BI and
approximation error in~\cref{sec:evaluation-bi}.

\subsection{From One Centroid to Multiple Components}
\label{sec:method-components}

To describe within-group structure beyond a single BI value,
we partition $\mathcal S$ into $k$ nonempty, disjoint
components $\mathcal C_1,\ldots,\mathcal C_k$, for a chosen
$k\in\{1,\ldots,m\}$.
Each $\mathcal C_r$ indexes a subset of the response cloud,
with representative $\mu_r\in\mathbb R^d$.
Bregman clustering jointly selects the partition and representatives by minimizing the following objective~\citep{banerjee2005clustering}:
\begin{equation}
    \min_{\substack{
        \mathcal S=\bigsqcup_{r=1}^{k}\mathcal C_r,\;
        \mathcal C_r\ne\varnothing\\
        \mu_1,\ldots,\mu_k\in\mathbb R^d
    }}
    \frac{1}{m}\sum_{r=1}^{k}\sum_{i\in\mathcal C_r}
    D_{\Phi_G}\bigl(
        \widetilde\theta_i^{(m)}(\epsilon),\mu_r
    \bigr).
    \label{eq:bregman-clustering}
\end{equation}
For a fixed partition, the optimal representative $\mu_r$
is the mean of the response points in $\mathcal C_r$.
Its size-weighted displacement defines the
\emph{component contribution}
\begin{equation}
    \Delta\theta_r
    \coloneq\frac{|\mathcal C_r|}{m}(\mu_r-\theta_0)
    =-\epsilon G^{-1}\sum_{i\in\mathcal C_r}z_i.
    \label{eq:component-response-direction}
\end{equation}
Using the same $G$ and $\epsilon$ for all components,
any partition satisfies
\begin{equation}
    \sum_{r=1}^{k}\Delta\theta_r
    =\widehat\theta_{\mathcal S}(\epsilon)-\theta_0.
    \label{eq:component-reconstruction}
\end{equation}
The target contributions
$\nabla_\theta f(\theta_0)^\top\Delta\theta_r$
therefore also sum to standard GI.
The decomposition therefore reveals structure without changing the additive estimate.

\section{Improving Group Influence with Component Structure}
\label{sec:method-cluster-optimization}

The component decomposition in~\cref{eq:component-reconstruction}
expresses the standard parameter-response estimate as a sum
of component contributions with unit coefficients.
We keep these component directions fixed and optimize their
coefficients using a richer approximation of the
group-reweighting term.
The resulting parameter change remains in the span of these
directions but is no longer restricted to their original sum,
as illustrated in~\cref{fig:illustration}(c).

\subsection{Quadratic Correction in Component Space}
\label{sec:method-joint-e2}

We parameterize the group-induced change by
\begin{equation}
    \theta(\alpha)=\theta_0+
    \sum_{r=1}^{k}\alpha_r\Delta\theta_r,
    \qquad \alpha\in\mathbb R^k.
    \label{eq:cluster-optimization-space}
\end{equation}
Here $\theta(\mathbf1)=\widehat\theta_{\mathcal S}(\epsilon)$
is already the unique minimizer of~\cref{eq:pbrf-lin}.
Obtaining a different parameter response therefore requires
modifying the objective used to optimize $\alpha$.

In~\cref{eq:pbrf-lin}, the group-reweighting term
$\epsilon c_{\mathcal S}^\top(\theta-\theta_0)$
is linear in the parameter displacement and therefore
omits the curvature of the reweighted training losses.
We incorporate this curvature through a GGN-based
quadratic approximation, while leaving the common penalty
$D_{\Phi_G}(\theta,\theta_0)$ unchanged:
\begin{equation}
    Q^{\mathrm{E2}}_{\mathcal S}(\theta;\epsilon)
    \coloneq Q^{\mathrm{lin}}_{\mathcal S}(\theta;\epsilon)
    +\frac{\epsilon}{2}\sum_{i\in\mathcal S}
    (\theta-\theta_0)^\top J_i^\top H_iJ_i
    (\theta-\theta_0).
    \label{eq:cluster-e2-objective}
\end{equation}
Here $J_i$ and $H_i$ are evaluated at $\theta_0$, as in~\cref{sec:prelim-if}, and removal uses $\epsilon=-1/N$.
We refer to this correction as E2. %
After substituting~\cref{eq:cluster-optimization-space}, the added term includes curvature along each component direction and cross terms between directions, evaluated over the entire reweighted group.

In \emph{Joint E2}, we jointly optimize the component
coefficients by minimizing $Q^{\mathrm{E2}}_{\mathcal S}$:
\begin{equation}
    \alpha^\star\in\argmin_{\alpha\in\mathbb R^k}
    Q^{\mathrm{E2}}_{\mathcal S}\bigl(\theta(\alpha);\epsilon\bigr).
    \label{eq:joint-e2-coefficients}
\end{equation}
The resulting target-change prediction is
$\nabla_\theta f(\theta_0)^\top(\theta(\alpha^\star)-\theta_0)$.
For $k=1$, coefficient fitting can only rescale the original group direction; multiple components allow their relative contributions to change.

\subsection{Efficient Approximations}
\label{sec:method-efficient}
\label{sec:method-intra-e2}
\label{sec:method-clustering-geometry}

Joint E2 requires both component construction and evaluation of the GGN curvature of the reweighted group.
We consider two approximations that reduce these costs separately. %

\paragraph{Within-component curvature approximation.}
For each reweighted example, Joint E2 evaluates the
GGN-based quadratic correction using the combined displacement
$\theta(\alpha)-\theta_0=\sum_{r=1}^{k}\alpha_r\Delta\theta_r$.
For an example in $\mathcal C_r$, \emph{Intra E2} instead
uses only $\alpha_r\Delta\theta_r$ in this additional
quadratic term, while leaving
$Q^{\mathrm{lin}}_{\mathcal S}(\theta(\alpha);\epsilon)$
unchanged:
\begin{equation}
    Q^{\mathrm{intra}}_{\mathcal S}(\alpha;\epsilon)
    \coloneq Q^{\mathrm{lin}}_{\mathcal S}\bigl(\theta(\alpha);\epsilon\bigr)
    +\frac{\epsilon}{2}\sum_{r=1}^{k}\alpha_r^2
    \sum_{i\in\mathcal C_r}\Delta\theta_r^\top J_i^\top H_iJ_i\Delta\theta_r.
    \label{eq:intra-e2-objective}
\end{equation}
This restriction omits cross-direction terms and evaluates each direction's curvature only on its own component's examples.
The unchanged common quadratic term retains cross-component terms, so we still optimize all coefficients jointly.

\paragraph{Euclidean component construction.}
To reduce clustering cost in \Cref{eq:bregman-clustering}, we form components by clustering
the raw gradients $z_i$ using Euclidean distances,
without curvature weighting.
This simplification applies only to clustering; we retain
the original $G$ when computing component directions
and optimizing their coefficients.

\section{Experiments}
\label{sec:experiments}
\suppressfloats[t]

We investigate whether projected BI is associated with the approximation error of standard GI and whether the proposed component-based correction improves estimation accuracy.
We first describe the common experimental setup in~\cref{sec:evaluation-setup} and then examine the BI-error relationship in~\cref{sec:evaluation-bi}.
Finally, we analyze component-wise influence patterns and evaluate the proposed corrections in~\cref{sec:evaluation-components}.

\subsection{Experimental Setup}
\label{sec:evaluation-setup}

\paragraph{Models and removal groups.}
We use logistic regression on Dogfish~\citep{koh2017understanding, koh2019accuracy}, a multilayer perceptron (MLP) on MNIST~\citep{bae2022if}, and ResNet9 on CIFAR-10~\citep{park2023trak}.
For each dataset, we fix the original model $\theta_0$ and select $100$ held-out test examples as \emph{queries}, as defined in~\cref{sec:prelim-if}.
For each removal ratio, we sample $100$ training groups uniformly without replacement within each group, allowing overlap between groups.
We predict how each group's removal changes the loss on each individual query.
All methods use the same removal groups and queries.

\paragraph{Targets and reference responses.}

Query $(x'_j,y'_j)$ defines the single-example loss $f_j(\theta)\coloneq\ell(h_\theta(x'_j),y'_j)$.
For a fixed query $j$, we write $f=f_j$ and suppress the query index.
With removal weight $\epsilon=-1/N$, we numerically optimize the PBRF objective in~\cref{eq:pbrf} once per group and reuse the resulting parameters across methods and queries.
The reference loss change for the query under evaluation is
\[
    A^{\mathrm{PBRF}}_{\mathcal S}(\epsilon)
    \coloneq
    f\bigl(\theta^{\mathrm{PBRF}}_{\mathcal S}(\epsilon)\bigr)
    -f(\theta_0).
\]
We use the full-training-set damped GGN matrix $G$, with the same ridge coefficient in $G$ and the PBRF proximal term.

\paragraph{Evaluation metrics.}
For each method and query, we compute Spearman correlation and raw $R^2$ from $100$ pairs of predicted and reference loss changes, one per removal group.
We repeat this evaluation for the $100$ queries and report the median of each metric over the common retained subset.
High rank correlation can coexist with large prediction errors~\citep{koh2019accuracy}.
We therefore use raw $R^2$ alongside Spearman to assess how closely predicted influence values match the reference, accounting for scale differences across queries.
Evaluation details are provided in~\cref{app:experimental-details}.

\begin{wrapfigure}{r}{0.5\textwidth}  %
    \centering
    \includegraphics[width=\linewidth]{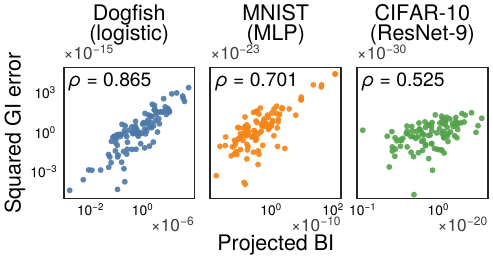}
    \caption{Bregman information tracks group influence error.
    Each point represents a random 1\% removal group (100 groups per model).
    We plot projected Bregman information against the squared error of GI relative to the actual PBRF change.
    For each model, we show the query whose Spearman correlation $\rho$ is closest to the median over 100 test queries.}
    \label{fig:projected-bi-error-scatter}
\end{wrapfigure}

\subsection{Bregman Information and Approximation Error}
\label{sec:evaluation-bi}

BI characterizes dispersion omitted by the centroid representation, but its association with nonlinear approximation error is an empirical question.
Using the $1\%$ groups for all three models, we compute, for each fixed query, the Spearman correlation across removal groups between projected BI and the squared error of standard GI,
\begin{equation}
    \mathcal E_{\mathcal S}(f,\epsilon)
    \coloneq\bigl[A^{\mathrm{PBRF}}_{\mathcal S}(\epsilon)
                  -\widehat A_{\mathcal S}(\epsilon)\bigr]^2,
    \label{eq:bi-squared-pbrf-error}
\end{equation}
where $\widehat A_{\mathcal S}(\epsilon)$ is defined in~\cref{eq:standard-group-influence}.
We plot the summed quantity $m\mathcal B_{\mathrm{proj}}(\mathcal S)$ without the removal factor $N^{-2}$, following the convention detailed in~\cref{app:diagnostic-details}.
Since the group size is fixed within each model, this rescaling does not affect the Spearman correlation.
For each model, we display the query whose correlation is closest to the median across the $100$ fixed queries.
\Cref{fig:projected-bi-error-scatter} shows a positive rank association between projected BI and squared GI error in all three settings.
This result empirically links the dispersion omitted by the centroid representation to GI approximation error.

\Cref{tab:geometry-ablation} compares projected BI with full BI and Euclidean gradient variance.
Projected BI exhibits the strongest correlation with squared GI error, whereas full BI and Euclidean gradient variance perform similarly.
This gap shows that target projection provides additional information for the target-specific approximation error considered here.
The comparison concerns dispersion measures; we evaluate the effect of clustering geometry on correction accuracy separately in~\cref{sec:evaluation-components}.

\subsection{Component Decomposition and Influence Correction}
\label{sec:evaluation-components}

We first examine the component structure hidden by GI, then evaluate whether reweighting these components improves prediction and how the efficient approximations affect accuracy and cost.

\paragraph{Component-wise influence patterns.}
We partition each of the $100$ random MNIST $10\%$ removal groups into $k=2$ Bregman components.
For each of the $100$ fixed queries, we compare the component contributions with their sum, which recovers GI by~\cref{eq:component-reconstruction}.
Component contributions are totals over the training examples in each component, rather than per-example averages.
To examine class-wise patterns, we average each component's influence over test examples within each class in the $10{,}000$-example MNIST test set, forming a ten-dimensional profile.

\begin{wrapfigure}{r}{0.5\textwidth}  %
  \centering
  \includegraphics[width=\linewidth]{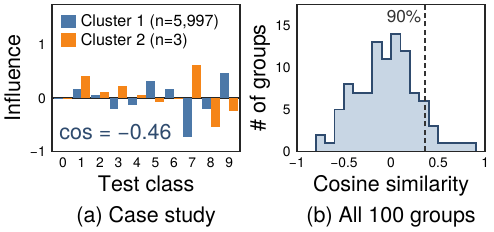}
  \caption{The two clusters tend to exhibit different class-wise influence patterns. For an MNIST MLP, we split each of 100 random 10\% removal groups
into $k=2$ Bregman clusters.
Each cluster's class-wise influence vector contains its contribution to GI
averaged over test points within each class.
\textbf{(a)}~Case study: unit-normalized class-wise influence vectors of the two clusters.
\textbf{(b)}~Cosine similarities across all 100 groups
(median $-0.03$, IQR $[-0.25, 0.17]$).
The dashed line marks the 90th percentile ($0.36$).}
\vskip -2em
  \label{fig:cluster-seperation}
\end{wrapfigure}

\cref{fig:cluster-seperation}(a) shows that the two components capture distinct class-wise influence patterns within a single removal group, with opposite-signed contributions for several classes and a cosine similarity of $-0.46$ between their unit-normalized profiles.
\cref{fig:cluster-seperation}(b) shows that the component profiles are generally not aligned across the $100$ removal groups, with a median cosine similarity of $-0.03$ and a 90th percentile of $0.36$.
The ablation in~\cref{app:omission-results} further shows that retaining both components yields higher median raw $R^2$ 
against the full-group PBRF reference than retaining either alone.

\paragraph{Accuracy of corrected estimates.}
\Cref{fig:actual-vs-pred-models} compares GI, single-component E2 ($k=1$), and Euclidean Joint and Intra E2 at $k=2$ across the three models.
Single-component E2 optimizes one coefficient along the original GI direction, whereas the two-component methods reweight the component contributions separately.
Although GI already ranks groups accurately, consistent with~\citet{koh2019accuracy}, single-component E2 improves raw $R^2$ over GI on all three models.
Joint E2 further raises median raw $R^2$ from $0.891$ to $0.945$ on Dogfish and from $0.965$ to $0.975$ on ResNet9; Intra E2 reaches $0.945$ and $0.974$, respectively.
On MNIST, where single-component E2 already achieves $0.993$, Joint and Intra E2 reach $0.995$ and $0.994$.
The paired significance tests in~\cref{app:diagnostic-details} show that both two-component variants improve query-wise Spearman correlation and raw $R^2$ over single-component E2 in all three settings, with $p_{\mathrm{adj}}<0.001$ throughout.
Thus, separately reweighting the component contributions improves on rescaling the aggregate direction alone.

\begin{figure}[t]  %
  \centering
  \includegraphics[width=0.95\linewidth]{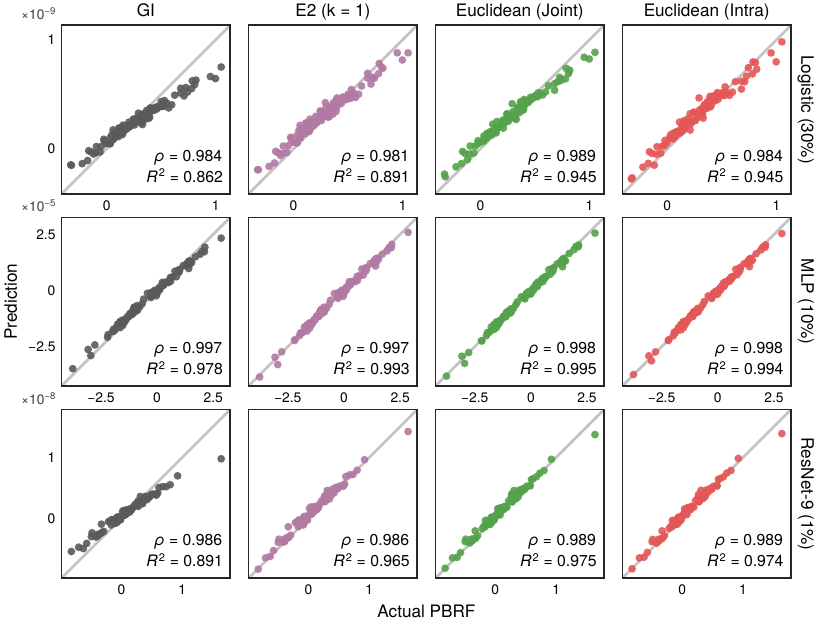}  
    \caption{
    Actual PBRF vs.\ predicted test-loss changes.
    From top to bottom, the rows show Dogfish logistic regression ($30\%$ removal), MNIST MLP ($10\%$), and CIFAR-10 ResNet9 ($1\%$).
    From left to right, the columns compare GI, single-component E2 ($k=1$), Euclidean Joint E2 ($k=2$), and Euclidean Intra E2 ($k=2$).
    Each panel shows one query across $100$ removal groups, while annotations report median query-wise Spearman $\rho$ and raw $R^2$ over the retained queries.
    The gray line marks $y=x$.
    Across all three models, the improvements of both $k=2$ variants over single-component E2 are significant at the $0.001$ level for both metrics after Holm correction.
    Statistical details are provided in~\cref{app:statistical-testing}.
    }
  \label{fig:actual-vs-pred-models}
\end{figure}

\begin{table}[t]
  \centering
    \caption{
    Ablation of clustering geometry and edit-curvature scope at $k=2$.
    For each query, Spearman $\rho$ and raw $R^2$ are computed across $100$ random removal groups.
    Entries report the median over the common retained queries, with the interquartile range in brackets.
    }
  \label{tab:clustering-curvature}
  \vspace{0.5em}
  \small
  \setlength{\tabcolsep}{4pt}
  \begin{tabular}{llcccc}
  \toprule
   & & \multicolumn{2}{c}{Dogfish (logistic, 30\%)} & \multicolumn{2}{c}{MNIST (MLP, 10\%)} \\
  \cmidrule(lr){3-4}\cmidrule(lr){5-6}
   & Clustering & Spearman $\rho\uparrow$ & $R^2\uparrow$ & Spearman $\rho\uparrow$ & $R^2\uparrow$ \\
  \midrule
  \multirow{2}{*}{\rotatebox[origin=c]{90}{Joint}}
    & Bregman
    & $0.989_{[0.982,\,0.991]}$ & $0.945_{[0.926,\,0.957]}$
    & $0.998_{[0.997,\,0.998]}$ & $0.995_{[0.993,\,0.996]}$ \\
    & Euclidean
    & $0.989_{[0.983,\,0.991]}$ & $0.945_{[0.926,\,0.958]}$
    & $0.998_{[0.997,\,0.998]}$ & $0.995_{[0.993,\,0.996]}$ \\
  \midrule
  \multirow{2}{*}{\rotatebox[origin=c]{90}{Intra}}
    & Bregman
    & $0.984_{[0.977,\,0.989]}$ & $0.945_{[0.919,\,0.957]}$
    & $0.998_{[0.997,\,0.998]}$ & $0.993_{[0.991,\,0.995]}$ \\
    & Euclidean
    & $0.984_{[0.976,\,0.989]}$ & $0.945_{[0.919,\,0.957]}$
    & $0.998_{[0.997,\,0.998]}$ & $0.994_{[0.992,\,0.995]}$ \\
  \bottomrule
  \end{tabular}
\end{table}

\newpage

\begin{wraptable}{r}{0.5\linewidth}
\centering
\caption{Measured stage-wise cost reductions on MNIST at $k=2$.
Baseline denotes Bregman clustering for partition construction and Joint E2 for
edit curvature; Approx. denotes Euclidean clustering and Intra E2, respectively.
Reduction is the relative decrease from the baseline stage time.
Big-O analysis and timing protocols are provided in~\cref{app:e2-complexity}.}
\vspace{0.5em}
\label{tab:efficient-measured-cost}
\resizebox{\linewidth}{!}{%
\begin{tabular}{@{}llrrr@{}}
\toprule
Stage & Removal ratio & Baseline & Approx. & Reduction \\
\midrule
Partition & $1\%$  & $59.4$s  & $0.277$s & $99.5\%$ \\
Edit curvature & $10\%$ & $0.193$s & $0.103$s & $46.6\%$ \\
\bottomrule
\end{tabular}%
}
\end{wraptable}

\paragraph{Efficient approximations.}
\Cref{tab:clustering-curvature} separates the effects of clustering geometry and edit-curvature scope by comparing Bregman and Euclidean clustering under Joint and Intra E2 at $k=2$ on Dogfish and MNIST.
Changing the clustering geometry has little effect under either curvature scope, and Intra E2 remains close to Joint E2 for either geometry.
The empirical evidence and attribution precedents supporting the Euclidean approximation are discussed in~\cref{app:geometry-ablation}.

\Cref{tab:efficient-measured-cost} reports the corresponding measured stage-wise cost reductions on MNIST: Euclidean clustering substantially reduces partition-construction time, while Intra E2 reduces edit-curvature time with
fixed Euclidean partitions.
These results identify computational savings in the replaced stages, alongside similar prediction accuracy in the evaluated $k=2$ settings.
Complexity analysis and timing details are provided in~\cref{app:e2-complexity}.

\newpage
\section{Conclusion}
\label{sec:conclusion}

We show that standard group influence compresses member-wise responses into a single representative, hiding variation within the group.
The Bregman view makes this compression explicit: BI quantifies the omitted variation, while multiple components preserve the original estimate and reveal how different parts of the group reinforce or oppose one another.
Experiments show that target-projected BI is associated with approximation error relative to the nonlinear PBRF, and that adjusting the component weights improves prediction of target changes.
The Bregman characterization is exact for the linearized response, whereas the BI--error association and estimation gains are empirical.
Together, these results show that structure hidden by additive aggregation can be exposed and used to refine group influence estimates.

\section*{AI Use Statement}
We used generative AI tools to support research ideation and execution, including refining experimental designs, implementing and debugging experimental and analysis code, and interpreting experimental results.
We also used these tools to draft and revise sections of the manuscript and to improve the clarity and readability of the writing.
The authors made the final decisions regarding the methodology, interpretation of results, and presentation, and take responsibility for the final content of this work, including all AI-assisted text, code, analyses, and other artifacts.

\section*{Ethics Statement}
This work studies methods for interpreting and estimating the effect of training-data removal. The experiments use standard benchmark datasets and fixed research models, and do not involve human-subject experiments, newly collected personal data, or deployment in real-world decision systems. The proposed method can support beneficial uses such as data auditing, debugging, and understanding when additive group influence hides within-group structure. However, influence estimates are approximations and should not be treated as definitive evidence about individual examples, especially in sensitive or consequential settings. Practitioners should use appropriate privacy safeguards, access control, and domain expert oversight when applying such analyses to sensitive data.

\section*{Reproducibility Statement}
We provide derivations and proofs in \cref{app:additional-theory} and detailed experimental protocols in \cref{app:experimental-details}, including dataset and model settings, removal-group construction, query selection, curvature construction, solver tolerances, clustering procedures, E2 coefficient solves, PBRF reference computation, timing protocols, and statistical testing. All main comparisons use fixed checkpoints, shared removal groups and queries, common curvature operators, and the same actual PBRF reference points within each setting.

\bibliography{iclr2027_conference}
\bibliographystyle{iclr2027_conference}

\appendix
\clearpage
\section{Related work}
\label[appendix]{app:related-work}

\paragraph{Training-data attribution.}
Training-data attribution explains model outputs in terms of the examples used for training; see \citet{hammoudeh2024training} for a survey.
Representative instance-based approaches include representer points, gradient-similarity methods, TracIn, and influence functions~\citep{koh2017understanding,yeh2018representer,charpiat2019input,pruthi2020estimating,hanawa2021evaluation}.
Counterfactual attribution and valuation approaches include Data Shapley, Datamodels, and TRAK~\citep{ghorbani2019datashapley,jia2019efficient,ilyas2022datamodels,park2023trak}.

\paragraph{The PBRF as the reference response.}
Prior work has documented substantial gaps between neural-network influence estimates and literal retraining, as well as sensitivity to practical approximation choices~\citep{basu2020fragile, schioppa2023theoretical}.
\citet{bae2022if} show that practical influence estimates for neural networks approximate the proximal Bregman response function (PBRF) more closely than leave-one-out retraining.
The PBRF provides a nonlinear reference whose local approximation recovers the damped-GGN influence estimate.
Following \citet{bae2022if}, several studies have used the nonlinear PBRF response to evaluate influence approximations, including function-space approximations~\citep{pmlr-v202-dhawan23a}, language-model influence estimates~\citep{grosse2023studyinglargelanguagemodel}, inverse-curvature solvers~\citep{klochkov2024revisitinginversehessianvector}.
The PBRF has also been extended to graph edge edits~\citep{heo2025influence}.

\paragraph{Correcting additive group influence.}
The accuracy of influence, approximate leave-one/-\(k\)-out, and Newton-step estimators has also been studied under convexity and smoothness assumptions~\citep{giordano2019swiss,rad2018scalable, fisher2023influence, rubinstein2026accuracynewtonstepinfluence}.
\citet{koh2019accuracy} find that high rank correlations between additive group influence and retraining effects can coexist with large errors in the predicted magnitudes.
Their analysis uses a one-step Newton approximation that accounts for removal-induced curvature changes~\citep{pregibon1981logistic}.
\citet{basu2020second} derive second-order approximations of the parameter response under group reweighting, whereas \citet{heo2026interactionawareinfluencefunctionsgroup} retain the first-order parameter change and expand the target function to second order, yielding pairwise interaction terms.
Rescaled Influence Functions (RIF)~\citep{rubinstein2025rescaled} instead sum single-example Newton removal steps, preserving additivity of the estimated parameter change.
Recent theory gives finite-sample analyses of IF and the single Newton step (NS) for convex ERM and identifies regimes in which NS is substantially more accurate than IF~\citep{rubinstein2026accuracynewtonstepinfluence}.

Our decomposition initially preserves standard GI.
E2 then optimizes a GGN-based quadratic surrogate of the PBRF edit objective within the span of the component contributions, and can therefore be viewed as a reduced-space Newton-type correction.
The analysis of~\citet{rubinstein2026accuracynewtonstepinfluence} concerns a full-space Newton step for leave-group-out ERM, whereas E2 targets the PBRF and adjusts only the coefficients of component directions derived from the Bregman decomposition.
E2 also keeps the component directions and target linearization fixed, rather than expanding the parameter response as in~\citet{basu2020second} or the target as in~\citet{heo2026interactionawareinfluencefunctionsgroup}.
RIF computes single-example Newton corrections independently; E2 instead fits the component coefficients jointly for the entire removal group.

\paragraph{Bregman centroids and information.}
\citet{banerjee2005clustering} develop centroid-based clustering with Bregman divergences and characterize Bregman information as the minimum average divergence to a single representative.
We apply these established concepts to the scaled member-wise parameter responses induced by the linearized group PBRF.
Our objective identity connects the additive estimate to its centroid representation, the dispersion it omits, and a multi-component decomposition.
In the quadratic setting, target-projected BI is a scaled variance of individual influence scores; its association with nonlinear PBRF error is an empirical finding, not a consequence of the centroid identity.

\paragraph{Grouping and clustering for efficient attribution.}
Generalized Group Data Attribution~\citep{ley2024generalizedgroupdataattribution} attributes model properties to groups rather than individual training examples, trading attribution granularity for computational efficiency.
Its framework includes grouped versions of influence functions, TracIn~\citep{pruthi2020estimating}, and TRAK~\citep{park2023trak}.
Clustered Influence Functions~\citep{bado2026clustered} cluster training-gradient sketches, cache inverse-curvature responses for full-space cluster means, and answer subset queries by membership-weighted recombination.
Their error analysis separates clustering scatter from solver error when approximating a fixed linearized influence estimate.
We instead partition each specified removal group, obtain component contributions that recover its original GI estimate, and then optimize their combination to improve agreement with the nonlinear PBRF.
Our Euclidean and Intra variants simplify partition construction and edit-curvature computation within this group-specific correction, rather than constructing a global response cache.

\paragraph{Influence representations for structure discovery.}
InfEmbed~\citep{wang2023error} applies $k$-means to influence embeddings of test examples to discover underperforming slices with similar training-data explanations.
DeMix~\citep{deng2026demix} represents each training example by its influences across validation predictions and uses multi-label classification to identify erroneous samples and their error types.
Our construction instead starts from the influence objective itself: we prove that the linearized group PBRF objective equals a Bregman centroid objective, up to a parameter-independent constant (\cref{thm:group-bregman-clustering}).
This identity grounds the response points and their clustering geometry in the objective underlying standard GI.
The multi-component extension refines this representation in the same geometry, with size-weighted contributions that exactly recover the additive estimate (\cref{sec:method-components}).
These components support both interpretation of the aggregate and its subsequent correction.

\newpage

\section{Additional Theory and Proofs}
\label[appendix]{app:additional-theory}

We first give the local derivations used in~\cref{sec:prelim}, then prove the centroid and dispersion identities in~\cref{sec:method} and the contribution decomposition in~\cref{sec:method-components}.
The final subsection records the additional conditions needed for a nonquadratic Bregman generator.

\subsection{Classical group influence by implicit differentiation}
\label[appendix]{app:classical-group-influence}

We give the group version of the classical derivation in~\citet{koh2017understanding}, with the fixed-target additivity discussed by~\citet{koh2019accuracy}.
Assume that \(\mathcal J\) and the losses \(L_i\), \(i\in \mathcal S\), are twice continuously differentiable on a neighborhood of \(\theta_0\), with
\[
    \nabla_\theta\mathcal J(\theta_0)=0,
    \qquad H=\nabla_\theta^2\mathcal J(\theta_0)\succ0.
\]
These are local assumptions at an interior minimizer, not properties required of every trained original model.
Define the stationarity map
\[
    F_{\mathcal S}(\theta,\epsilon)
    :=\nabla_\theta\mathcal J(\theta)
      +\epsilon\sum_{i\in \mathcal S}\nabla_\theta L_i(\theta).
\]
Since \(F_{\mathcal S}(\theta_0,0)=0\) and \(D_\theta F_{\mathcal S}(\theta_0,0)=H\), the implicit-function theorem gives a unique continuously differentiable stationary branch \(\theta_{\mathcal S}^{\mathrm{loc}}(\epsilon)\) near \((\theta_0,0)\).
Continuity of the Hessian makes this a branch of strict local minimizers for sufficiently small \(|\epsilon|\).

Applying the chain rule to \(F_{\mathcal S}(\theta_{\mathcal S}^{\mathrm{loc}}(\epsilon),\epsilon)=0\) gives
\begin{equation}
    D_\theta F_{\mathcal S}(\theta_{\mathcal S}^{\mathrm{loc}}(\epsilon),\epsilon)
    \frac{d\theta_{\mathcal S}^{\mathrm{loc}}(\epsilon)}{d\epsilon}
    +\partial_\epsilon F_{\mathcal S}(\theta_{\mathcal S}^{\mathrm{loc}}(\epsilon),\epsilon)=0.
    \label{eq:appendix-classical-stationarity-chain-rule}
\end{equation}
At \(\epsilon=0\), the two derivatives of \(F_{\mathcal S}\) are \(H\) and \(c_{\mathcal S}=\sum_{i\in \mathcal S}z_i\), respectively, so
\[
    \left.\frac{d\theta_{\mathcal S}^{\mathrm{loc}}(\epsilon)}{d\epsilon}\right|_{\epsilon=0}
    =-H^{-1}c_{\mathcal S}.
\]
Consequently, for fixed \(\mathcal S\),
\begin{equation*}
    \theta_{\mathcal S}^{\mathrm{loc}}(\epsilon)
    =\theta_0-\epsilon H^{-1}c_{\mathcal S}+o(|\epsilon|),
    \qquad \epsilon\to0.
\end{equation*}
If the perturbed objectives are globally strongly convex and attain their minima for sufficiently small \(|\epsilon|\), this branch is their unique global minimizer \(\theta_{\mathcal S}^\star(\epsilon)\).
Otherwise, the conclusion is local.

For a scalar target \(f\) differentiable at \(\theta_0\), a second application of the chain rule yields
\[
    \left.\frac{d}{d\epsilon}f(\theta_{\mathcal S}^{\mathrm{loc}}(\epsilon))\right|_{\epsilon=0}
    =\nabla_\theta f(\theta_0)^\top
      \left.\frac{d\theta_{\mathcal S}^{\mathrm{loc}}(\epsilon)}{d\epsilon}\right|_{\epsilon=0}
    =-\nabla_\theta f(\theta_0)^\top H^{-1}c_{\mathcal S}.
\]
Thus
\begin{equation*}
    f(\theta_{\mathcal S}^{\mathrm{loc}}(\epsilon))-f(\theta_0)
    =-\epsilon\nabla_\theta f(\theta_0)^\top H^{-1}c_{\mathcal S}+o(|\epsilon|).
\end{equation*}
The first-order term is additive over disjoint groups for a fixed original model and target.
Setting \(\epsilon=-1/N\) in the parameter expansion recovers the removal estimate in~\cref{eq:classical-if}; the local derivation alone does not guarantee accuracy at that finite weight.

\subsection{PBRF influence by the chain rule and its quadratic objective}
\label[appendix]{app:pbrf-linearization}

We derive the first-order response of~\cref{eq:pbrf} by differentiating its stationarity condition, then identify the quadratic objective in~\cref{eq:pbrf-lin}.
This is the group version of the PBRF construction of~\citet{bae2022if}, with the negative removal weight used throughout this paper.
Assume that \(h_\theta(x_v)\) is twice continuously differentiable near \(\theta_0\), and that \(\ell(\cdot,y_v)\) is twice continuously differentiable near \(h_{\theta_0}(x_v)\) and convex on a convex domain containing the attainable outputs.
Write \(h_v(\theta)=h_\theta(x_v)\in\mathbb R^C\), where \(C\) is the output dimension, and set
\[
    h_v^s=h_v(\theta_0),\qquad
    J_v(\theta)=D_\theta h_v(\theta)\in\mathbb R^{C\times d},\qquad
    H_v(\theta)=\nabla_u^2\ell(h_v(\theta),y_v)\in\mathbb R^{C\times C}.
\]
Thus \(J_v=J_v(\theta_0)\) and \(H_v=H_v(\theta_0)\) agree with the notation in~\cref{sec:prelim-if}.

\paragraph{Source derivatives of the output-Bregman term.}
Define
\[
    B(\theta):=\frac1N\sum_{v=1}^N b_v(\theta),
    \qquad b_v(\theta):=D_{\ell(\cdot,y_v)}(h_v(\theta),h_v^s),
\]
and let \(e_v(\theta)=\nabla_u\ell(h_v(\theta),y_v)-\nabla_u\ell(h_v^s,y_v)\).
The chain and product rules give the exact identities
\begin{align*}
    \nabla_\theta b_v(\theta)
       &=J_v(\theta)^\top e_v(\theta),\nonumber\\
    \nabla_\theta^2 b_v(\theta)
       &=J_v(\theta)^\top H_v(\theta)J_v(\theta)
         +\sum_{a=1}^C e_{v,a}(\theta)\,
                \nabla_\theta^2 h_{v,a}(\theta).
\end{align*}
At \(\theta_0\), \(e_v(\theta_0)=0\), so the terms involving second derivatives of the model vanish exactly.
Therefore
\begin{equation*}
    B(\theta_0)=0,\qquad
    \nabla_\theta B(\theta_0)=0,\qquad
    \nabla_\theta^2B(\theta_0)=G_0:=\frac1N\sum_{v=1}^N J_v^\top H_vJ_v.
\end{equation*}
In particular, \(G_0\) is the exact Hessian of the output-Bregman term at \(\theta_0\), not an approximation obtained by discarding nonzero model-Hessian terms there.
It need not equal the Hessian of the original training loss.
Since \(H_v\succeq0\) and \(\lambda>0\), \(G=G_0+\lambda I_d\succ0\).

\paragraph{Differentiating the minimizing branch.}
At \(\epsilon=0\), nonnegativity of the output divergences and the positive proximal penalty make \(\theta_0\) the unique global PBRF minimizer.
This requires no condition \(\nabla_\theta\mathcal J(\theta_0)=0\).
Let
\[
    F_{\mathcal S}^{\mathrm{PBRF}}(\theta,\epsilon)
    :=\nabla_\theta B(\theta)+\lambda(\theta-\theta_0)
       +\epsilon\sum_{i\in \mathcal S}\nabla_\theta L_i(\theta).
\]
Its derivatives at the unperturbed point are
\[
    D_\theta F_{\mathcal S}^{\mathrm{PBRF}}(\theta_0,0)=G,
    \qquad \partial_\epsilon F_{\mathcal S}^{\mathrm{PBRF}}(\theta_0,0)=c_{\mathcal S}.
\]
The implicit-function theorem gives a continuously differentiable branch of strict local minimizers \(\theta_{\mathcal S}^{\mathrm{PBRF,loc}}(\epsilon)\) through \(\theta_0\).
Differentiating \(F_{\mathcal S}^{\mathrm{PBRF}}(\theta_{\mathcal S}^{\mathrm{PBRF,loc}}(\epsilon),\epsilon)=0\) by the chain rule, as in~\cref{eq:appendix-classical-stationarity-chain-rule}, gives
\begin{equation*}
    G\left.\frac{d\theta_{\mathcal S}^{\mathrm{PBRF,loc}}(\epsilon)}{d\epsilon}\right|_{\epsilon=0}
    +c_{\mathcal S}=0,
    \qquad
    \left.\frac{d\theta_{\mathcal S}^{\mathrm{PBRF,loc}}(\epsilon)}{d\epsilon}\right|_{\epsilon=0}
    =-G^{-1}c_{\mathcal S}.
\end{equation*}
The group-Hessian term \(\epsilon\sum_{i\in \mathcal S}\nabla_\theta^2L_i(\theta)\) vanishes at \(\epsilon=0\); no assumption that it is identically zero is needed.
For fixed \(\mathcal S\),
\begin{equation}
    \theta_{\mathcal S}^{\mathrm{PBRF,loc}}(\epsilon)
    =\theta_0-\epsilon G^{-1}c_{\mathcal S}+o(|\epsilon|),
    \qquad \epsilon\to0.
    \label{eq:appendix-pbrf-local-response}
\end{equation}
For any scalar target \(f\) differentiable at \(\theta_0\), the chain rule then gives
\begin{align*}
    \left.\frac{d}{d\epsilon}
      f(\theta_{\mathcal S}^{\mathrm{PBRF,loc}}(\epsilon))\right|_{\epsilon=0}
    &=\nabla_\theta f(\theta_0)^\top
      \left.\frac{d\theta_{\mathcal S}^{\mathrm{PBRF,loc}}(\epsilon)}{d\epsilon}\right|_{\epsilon=0}\nonumber\\
    &=-\nabla_\theta f(\theta_0)^\top G^{-1}c_{\mathcal S}
     =-q^\top c_{\mathcal S},
    \quad q=G^{-1}\nabla_\theta f(\theta_0).
\end{align*}
Multiplying by \(\epsilon\) gives \(\widehat A_{\mathcal S}(\epsilon)\) in~\cref{eq:standard-group-influence}.
This does not identify the first-order estimate with the nonlinear target change at a finite removal weight.

\paragraph{The quadratic objective for the first-order response.}
The affine approximation in~\cref{eq:appendix-pbrf-local-response} is the unique global minimizer of
\[
    Q_{\mathcal S}^{\mathrm{lin}}(\theta;\epsilon)
    =\tfrac12(\theta-\theta_0)^\top G(\theta-\theta_0)
       +\epsilon c_{\mathcal S}^\top(\theta-\theta_0),
\]
because its stationarity equation is \(G(\theta-\theta_0)+\epsilon c_{\mathcal S}=0\).
Thus implicit differentiation and the quadratic objective in~\cref{eq:pbrf-lin} give the same damped-GGN estimate.

The objective approximation can also be checked directly.
For \(\delta=\theta-\theta_0\),
\begin{align*}
    B(\theta_0+\delta)&=\tfrac12\delta^\top G_0\delta+o(\|\delta\|_2^2),\nonumber\\
    \sum_{i\in \mathcal S}L_i(\theta_0+\delta)
    &=\sum_{i\in \mathcal S}L_i(\theta_0)+c_{\mathcal S}^\top\delta+O(\|\delta\|_2^2).
\end{align*}
Hence, as \(\delta\to0\) with \(\mathcal S\) fixed,
\begin{align*}
    Q_{\mathcal S}^{\mathrm{PBRF}}(\theta_0+\delta;\epsilon)
    ={}&\epsilon\sum_{i\in \mathcal S}L_i(\theta_0)
       +Q_{\mathcal S}^{\mathrm{lin}}(\theta_0+\delta;\epsilon)
       +R_{\mathcal S}(\delta,\epsilon),\nonumber\\
    R_{\mathcal S}(\delta,\epsilon)
    ={}&o(\|\delta\|_2^2)+O(|\epsilon|\|\delta\|_2^2).
\end{align*}
This retains the unperturbed Bregman term to second order and the group perturbation to first order in \(\delta\).
It is not the full second-order Taylor objective at a fixed nonzero \(\epsilon\), which would also contain the group-Hessian term.
The centroid characterization in~\cref{sec:method} is exact for \(Q_{\mathcal S}^{\mathrm{lin}}\); the local argument does not guarantee a unique global nonlinear PBRF minimizer at finite removal weights.
For piecewise-smooth networks, this differentiation argument applies on neighborhoods where the required derivatives exist.
The algebraic centroid identity itself only requires the specified gradients and \(G\succ0\).

\subsection{Bregman centroids and proof of the main characterization}
\label[appendix]{app:proof-group-clustering}

We first state the centroid identity underlying the definition of Bregman information~\citep[Proposition~1]{banerjee2005clustering}.

\begin{lemma}[Fixed-group Bregman centroid identity]
\label{lem:fixed-group-centroid-identity}
Let \(\Psi\) be differentiable and convex on an open convex set \(\Omega\), and let \(a_1,\ldots,a_m\in\Omega\), with \(m\ge1\) and \(\bar a=m^{-1}\sum_i a_i\).
For every \(\mu\in\Omega\),
\begin{equation}
    \frac1m\sum_{i=1}^mD_\Psi(a_i,\mu)
    =\frac1m\sum_{i=1}^mD_\Psi(a_i,\bar a)+D_\Psi(\bar a,\mu).
    \label{eq:appendix-centroid-identity}
\end{equation}
Hence \(\bar a\) is a minimizer. It is unique if \(\Psi\) is strictly convex.
\end{lemma}
\begin{proof}
Since \(\sum_i(a_i-\bar a)=0\), the difference between the first two terms in~\cref{eq:appendix-centroid-identity} is
\[
    \Psi(\bar a)-\Psi(\mu)-\nabla\Psi(\mu)^\top(\bar a-\mu)
    =D_\Psi(\bar a,\mu).
\]
Convexity makes this quantity nonnegative; strict convexity makes it zero only when \(\mu=\bar a\).
\end{proof}

\begin{proof}[Proof of~\cref{thm:group-bregman-clustering}]
Let \(\delta=\theta-\theta_0\) and use the response points from~\cref{eq:individual-response-centers}.
Completing the square gives
\begin{align*}
    D_{\Phi_G}\!\left(\theta,\widetilde\theta_i^{(m)}(\epsilon)\right)
    &=\tfrac12(\delta+\epsilon mG^{-1}z_i)^\top
                   G(\delta+\epsilon mG^{-1}z_i)\nonumber\\
    &=\tfrac12\delta^\top G\delta+\epsilon m z_i^\top\delta
      +\tfrac{\epsilon^2m^2}{2}z_i^\top G^{-1}z_i.
\end{align*}
Averaging over \(\mathcal S\) yields
\[
    \frac1m\sum_{i\in \mathcal S}D_{\Phi_G}\!\left(\theta,\widetilde\theta_i^{(m)}(\epsilon)\right)
    =Q_{\mathcal S}^{\mathrm{lin}}(\theta;\epsilon)
      +\frac{\epsilon^2m}{2}\sum_{i\in \mathcal S}z_i^\top G^{-1}z_i.
\]
Rearranging and using the symmetry of the quadratic divergence gives~\cref{eq:group-clustering-objective} with the stated constant \(K_{\mathcal S}(\epsilon)\).
By~\cref{lem:fixed-group-centroid-identity}, its unique minimizer is
\[
    \frac1m\sum_{i\in \mathcal S}\widetilde\theta_i^{(m)}(\epsilon)
    =\theta_0-\epsilon G^{-1}\sum_{i\in \mathcal S}z_i
    =\widehat\theta_{\mathcal S}(\epsilon).
\]
The identity holds for every \(\theta\) and every real \(\epsilon\), including \(\epsilon=0\).
\end{proof}

\subsection{Full and target-projected Bregman information}
\label[appendix]{app:information-identities}

Set \(\bar z_{\mathcal S}=m^{-1}\sum_{i\in \mathcal S}z_i\).
The centered response points satisfy
\[
    \widetilde\theta_i^{(m)}(\epsilon)-\widehat\theta_{\mathcal S}(\epsilon)
    =-\epsilon mG^{-1}(z_i-\bar z_{\mathcal S}).
\]
Substitution into the quadratic divergence proves
\begin{align*}
    \frac1m\sum_{i\in \mathcal S}
    D_{\Phi_G}\!\left(\widetilde\theta_i^{(m)}(\epsilon),\widehat\theta_{\mathcal S}(\epsilon)\right)
    &=\frac{\epsilon^2m^2}{2m}\sum_{i\in \mathcal S}
      (z_i-\bar z_{\mathcal S})^\top G^{-1}(z_i-\bar z_{\mathcal S})\nonumber\\
    &=\epsilon^2m^2\mathcal B_G(\mathcal S).
\end{align*}
Thus the normalized full BI in~\cref{eq:full-bregman-information} removes only the common perturbation and group-size scale.
It still depends on the fixed source and curvature matrix \(G\).

For a differentiable scalar target \(f\), write \(g=\nabla f(\theta_0)\) and \(q=G^{-1}g\).
Define the linearized target value at each response point by
\[
    t_i^{\mathrm{lin}}
    :=g^\top\bigl(\widetilde\theta_i^{(m)}(\epsilon)-\theta_0\bigr)
    =-\epsilon m q^\top z_i.
\]
Its arithmetic mean is \(\widehat A_{\mathcal S}(\epsilon)\).
For \(\psi(t)=t^2/2\),
\begin{equation*}
    \frac1m\sum_{i\in \mathcal S}D_\psi(t_i^{\mathrm{lin}},\widehat A_{\mathcal S}(\epsilon))
    =\epsilon^2m^2\mathcal B_{\mathrm{proj}}(\mathcal S).
\end{equation*}
The singleton contribution uses weight \(\epsilon\), rather than \(m\epsilon\), and equals \(\widehat A_{\{i\}}(\epsilon)=-\epsilon q^\top z_i\).
Consequently,
\begin{equation*}
    \tfrac12\operatorname{Var}_{i\sim\mathrm{Unif}(\mathcal S)}
       \bigl(\widehat A_{\{i\}}(\epsilon)\bigr)
    =\epsilon^2\mathcal B_{\mathrm{proj}}(\mathcal S).
\end{equation*}
These are exact identities for linearized target values, not for nonlinear evaluations \(f(\widetilde\theta_i^{(m)}(\epsilon))-f(\theta_0)\).
The variance uses denominator \(m\), not \(m-1\).

\subsection{Component centroids, conservation, and whitening}
\label[appendix]{app:quadratic-bregman-duality}

We first show that component centroids preserve the original group estimate for any partition, as used in~\cref{eq:component-reconstruction}.
We then express the within-component Bregman objective in whitened coordinates.
Throughout, \(G\succ0\), \(m=|\mathcal S|\ge1\), and \(\mathcal C\) partitions \(\mathcal S\) into \(k\) nonempty components, with \(1\le k\le m\).

\paragraph{Conservation of the group estimate.}
For \(C\in\mathcal C\), define \(\pi_C=|C|/m\), \(\bar z_C=|C|^{-1}\sum_{i\in C}z_i\), and
\begin{equation}
    \widetilde\theta_C^{(m)}(\epsilon)
    :=\frac1{|C|}\sum_{i\in C}\widetilde\theta_i^{(m)}(\epsilon)
     =\theta_0-\epsilon mG^{-1}\bar z_C.
    \label{eq:cluster-response-centroids}
\end{equation}
For a fixed target, let \(g=\nabla_\theta f(\theta_0)\) and \(q=G^{-1}g\).
The component contribution is
\begin{equation*}
    \widehat A_C(\epsilon)
    =\pi_C g^\top\bigl(\widetilde\theta_C^{(m)}(\epsilon)-\theta_0\bigr)
    =-\epsilon q^\top\sum_{i\in C}z_i.
\end{equation*}
Since the components partition \(\mathcal S\), summing their gradient sums gives \(c_{\mathcal S}\), and therefore
\begin{equation*}
    \sum_{C\in\mathcal C}\widehat A_C(\epsilon)=\widehat A_{\mathcal S}(\epsilon),
    \qquad
    \sum_{C\in\mathcal C}\pi_C\widetilde\theta_C^{(m)}(\epsilon)
    =\widehat\theta_{\mathcal S}(\epsilon).
\end{equation*}
Thus contributions add, whereas centroids recover the group centroid by a size-weighted average.
Both identities hold for any partition; they do not depend on optimizing an assignment objective.
The centroid of a component in the original \(m\)-scaled cloud is not generally the standalone linearized PBRF estimate for that subgroup.
Their displacements are related by
\[
    \widehat\theta_C(\epsilon)-\theta_0
    =\pi_C\bigl(\widetilde\theta_C^{(m)}(\epsilon)-\theta_0\bigr).
\]
These identities concern the fixed additive estimate and do not imply additivity of nonlinear PBRF changes.

\paragraph{Choosing components in the same geometry.}
For a prescribed \(k\), the partition criterion in~\cref{eq:bregman-clustering} is
\begin{equation*}
    \min_{\substack{\mathcal C\text{ partitions }\mathcal S\\
                    |\mathcal C|=k}}
       W_{\mathcal C},
    \qquad
    W_{\mathcal C}:=\frac1m\sum_{C\in\mathcal C}\sum_{i\in C}
       D_{\Phi_G}\!\left(\widetilde\theta_i^{(m)}(\epsilon),
                         \widetilde\theta_C^{(m)}(\epsilon)\right).
\end{equation*}
For each fixed partition,~\cref{lem:fixed-group-centroid-identity} shows that its component means minimize the divergence over representatives.
Selecting components by this criterion refines the description of the cloud without changing the original group estimate.
It does not guarantee that the resulting partition maximizes cancellation visible to any particular target.

The quadratic generator and its conjugate are
\begin{equation*}
    \Phi_G(\theta)=\tfrac12\theta^\top G\theta,
    \qquad \Phi_G^\ast(z)=\tfrac12 z^\top G^{-1}z.
\end{equation*}
\begin{proposition}[Whitened response-space equivalence]
\label{prop:whitened-response-clustering}
Define
\begin{equation*}
    w_i:=G^{-1/2}z_i,
    \qquad \bar w_C:=\frac1{|C|}\sum_{i\in C}w_i,
\end{equation*}
where \(G^{-1/2}\) is the symmetric positive definite inverse square root.
For every partition \(\mathcal C\) into nonempty components,
\begin{equation*}
    W_{\mathcal C}
    =\frac{\epsilon^2m^2}{2m}\sum_{C\in\mathcal C}\sum_{i\in C}
       \|w_i-\bar w_C\|_2^2.
\end{equation*}
For \(\epsilon\ne0\), minimizing \(W_{\mathcal C}\) and the unscaled within-component Euclidean sum of squares gives the same globally optimal partitions.
\end{proposition}
\begin{proof}
From~\cref{eq:individual-response-centers,eq:cluster-response-centroids},
\begin{equation*}
    \widetilde\theta_i^{(m)}(\epsilon)-\widetilde\theta_C^{(m)}(\epsilon)
    =-\epsilon mG^{-1}(z_i-\bar z_C).
\end{equation*}
Substitution into the quadratic divergence gives
\begin{equation*}
    D_{\Phi_G}\!\left(\widetilde\theta_i^{(m)}(\epsilon),
                       \widetilde\theta_C^{(m)}(\epsilon)\right)
    =\frac{\epsilon^2m^2}{2}\|w_i-\bar w_C\|_2^2.
\end{equation*}
Summing proves the identity.
For nonzero \(\epsilon\), the scale factor is positive and independent of the partition.
For \(\epsilon=0\), all parameter-response points coincide and every partition minimizes \(W_{\mathcal C}\), so equivalence with the unscaled Euclidean optimization need not hold.
\end{proof}
The equivalence is a change of coordinates, not a guarantee that a particular \(k\)-means initialization finds a global optimum.

\paragraph{Within- and between-component dispersion.}
Let \(a_i=\widetilde\theta_i^{(m)}(\epsilon)\), \(a_C=\widetilde\theta_C^{(m)}(\epsilon)\), and \(\bar a=\widehat\theta_{\mathcal S}(\epsilon)\).
Applying~\cref{lem:fixed-group-centroid-identity} within each component and summing gives
\begin{equation*}
    \underbrace{\frac1m\sum_{i\in \mathcal S}D_{\Phi_G}(a_i,\bar a)}_{T}
    =\underbrace{\frac1m\sum_{C\in\mathcal C}\sum_{i\in C}D_{\Phi_G}(a_i,a_C)}_{W_{\mathcal C}}
     +\underbrace{\sum_{C\in\mathcal C}\pi_CD_{\Phi_G}(a_C,\bar a)}_{B_{\mathcal C}}.
\end{equation*}
All three dispersions share the same positive scaling under whitening when \(\epsilon\ne0\).
This decomposes the cloud's dispersion, not its nonlinear approximation error.

\subsection{Beyond quadratic generators: a dual-space centroid representation}
\label[appendix]{app:center-shift}

The preceding results use a quadratic generator.
For completeness, we state when a linear tilt of a general generator has a shifted-center representation.
The gradient and conjugacy properties used here are standard Legendre duality; see~\citet[Theorem~2 and Appendix~A]{banerjee2005clustering}.

Let \(\Phi:\mathbb R^d\to(-\infty,+\infty]\) be proper, closed, convex, and Legendre, with nonempty \(\Theta=\operatorname{int}(\operatorname{dom}\Phi)\).
Define its Fenchel conjugate \(\Phi^\ast(\eta)=\sup_\theta\{\eta^\top\theta-\Phi(\theta)\}\) and \(\Theta^\ast=\operatorname{int}(\operatorname{dom}\Phi^\ast)\).
Then
\begin{equation*}
    \nabla\Phi:\Theta\longrightarrow\Theta^\ast
    \quad\text{is bijective, with inverse }\nabla\Phi^\ast.
\end{equation*}

\begin{theorem}[Center shift under a linear tilt]
\label{thm:general-center-shift}
Fix \(\theta_0\in\Theta\), \(a\in\mathbb R^d\), and \(\eta=\nabla\Phi(\theta_0)-a\).
There exist \(\widetilde\theta\in\Theta\) and a constant \(K\) such that
\begin{equation}
    D_\Phi(\theta,\theta_0)+a^\top(\theta-\theta_0)
    =D_\Phi(\theta,\widetilde\theta)+K
    \quad\text{for every }\theta\in\Theta
    \label{eq:appendix-center-shift-equality}
\end{equation}
if and only if
\begin{equation*}
    \eta\in\Theta^\ast=\nabla\Phi(\Theta).
\end{equation*}
In that case the center and constant are unique:
\begin{equation}
    \widetilde\theta=\nabla\Phi^\ast(\eta),
    \qquad K=-D_\Phi(\theta_0,\widetilde\theta).
    \label{eq:appendix-center-shift-solution}
\end{equation}
Moreover,
\begin{equation*}
    D_\Phi(\theta,\widetilde\theta)
    =D_{\Phi^\ast}(\eta,\nabla\Phi(\theta)).
\end{equation*}
\end{theorem}
\begin{proof}
Differentiating~\cref{eq:appendix-center-shift-equality} in \(\theta\) forces
\(\nabla\Phi(\widetilde\theta)=\nabla\Phi(\theta_0)-a=\eta\).
This proves necessity and uniqueness of the center.
Conversely, if \(\eta\in\Theta^\ast\), the center in~\cref{eq:appendix-center-shift-solution} exists.
Expanding the divergences shows that the difference in~\cref{eq:appendix-center-shift-equality} is constant; evaluating at \(\theta=\theta_0\) determines \(K\).
Finally, \(\Phi^\ast(\nabla\Phi(v))=v^\top\nabla\Phi(v)-\Phi(v)\) gives
\(D_\Phi(u,v)=D_{\Phi^\ast}(\nabla\Phi(v),\nabla\Phi(u))\), proving the last identity.
\end{proof}

The range condition is substantive: \(\Phi(t)=e^t\) is Legendre on \(\mathbb R\), but its gradient range is \((0,\infty)\).
For \(\theta_0=0\) and \(a=2\), the required dual coordinate is \(-1\), so no interior shifted center exists.
For the positive definite quadratic generator, the gradient range is all of \(\mathbb R^d\).

\begin{theorem}[Centroid representation in dual coordinates]
\label{thm:general-group-bregman-representation}
With \(\Phi\) as above, fix \(\theta_0\in\Theta\), a nonempty group \(\mathcal S\) of size \(m\), gradients \(z_i\in\mathbb R^d\), and \(\epsilon\in\mathbb R\).
Suppose
\begin{equation*}
    \eta_i^{(m)}:=\nabla\Phi(\theta_0)-\epsilon m z_i\in\Theta^\ast
    \quad\text{for every }i\in \mathcal S.
\end{equation*}
Let
\begin{equation*}
    \widetilde\theta_i^{(m)}:=\nabla\Phi^\ast(\eta_i^{(m)}),
    \qquad \bar\eta_{\mathcal S}:=\frac1m\sum_{i\in \mathcal S}\eta_i^{(m)}
       =\nabla\Phi(\theta_0)-\epsilon c_{\mathcal S}.
\end{equation*}
For every \(\theta\in\Theta\),
\begin{align*}
    D_\Phi(\theta,\theta_0)+\epsilon c_{\mathcal S}^\top(\theta-\theta_0)
    &=\frac1m\sum_{i\in \mathcal S}D_\Phi(\theta,\widetilde\theta_i^{(m)})
       +K_{\mathcal S}^\Phi(\epsilon)\\ %
    &=\frac1m\sum_{i\in \mathcal S}D_{\Phi^\ast}(\eta_i^{(m)},\nabla\Phi(\theta))
       +K_{\mathcal S}^\Phi(\epsilon),%
\end{align*}
where
\begin{equation*}
    K_{\mathcal S}^\Phi(\epsilon):=-\frac1m\sum_{i\in \mathcal S}D_\Phi(\theta_0,\widetilde\theta_i^{(m)}).
\end{equation*}
The unique minimizer is
\begin{equation*}
    \widehat\theta_{\mathcal S}^\Phi(\epsilon)=\nabla\Phi^\ast(\bar\eta_{\mathcal S}).
\end{equation*}
\end{theorem}
\begin{proof}
Distribute the linear tilt as
\begin{equation*}
    D_\Phi(\theta,\theta_0)+\epsilon c_{\mathcal S}^\top(\theta-\theta_0)
    =\frac1m\sum_{i\in \mathcal S}
       \bigl[D_\Phi(\theta,\theta_0)+\epsilon m z_i^\top(\theta-\theta_0)\bigr].
\end{equation*}
Apply~\cref{thm:general-center-shift} to each summand with \(a=\epsilon m z_i\), then use its dual identity.
This proves both representations and their common constant.
Convexity of \(\Theta^\ast\) implies \(\bar\eta_{\mathcal S}\in\Theta^\ast\).
By~\cref{lem:fixed-group-centroid-identity}, the unique optimal representative in the dual objective is \(\nabla\Phi(\theta)=\bar\eta_{\mathcal S}\), giving the stated minimizer.
\end{proof}

For a nonquadratic generator, the arithmetic centroid is in dual coordinates, not generally in parameter space, and the divergence arguments cannot generally be reversed.
The member-wise representation also requires more than existence of an aggregate shifted center: the latter needs only \(\bar\eta_{\mathcal S}\in\Theta^\ast\), whereas the former needs every \(\eta_i^{(m)}\in\Theta^\ast\).
This is a statement about a Bregman divergence plus a linear tilt, not a claim that the full nonlinear PBRF has this form.

\paragraph{Quadratic specialization and scope.}
For \(\Phi_G(\theta)=\tfrac12\theta^\top G\theta\) with \(G\succ0\),
\(\Theta^*=\mathbb R^d\), \(\nabla\Phi_G(\theta)=G\theta\), and
\(\nabla\Phi_G^*(\eta)=G^{-1}\eta\).
The range condition is then automatic, the shifted points are
\(\theta_0-\epsilon mG^{-1}z_i\), and the dual-centroid formula recovers~\cref{thm:group-bregman-clustering}.
For a nonquadratic generator, averaging in dual coordinates does not in general commute with \(\nabla\Phi^*\).
Consequently, this extension does not establish the additive parameter-response decomposition in~\cref{eq:component-reconstruction} outside the quadratic setting, nor does it provide an accuracy guarantee for E2 or nonlinear PBRF.

\newpage

\section{Experimental Details}
\label[appendix]{app:experimental-details}

This section specifies the common protocol, numerical methods, and cost model.

\subsection{Common protocol}
\label[appendix]{app:common-setup}

Each comparison uses a fixed original checkpoint, $100$ random removal groups per removal ratio, and $100$ fixed test queries.
All compared methods use the same original model, removal groups, test queries, curvature operator, and actual PBRF reference points.
Partitions and coefficients are fitted using training-side quantities only and reused across queries.
Unless stated otherwise, a method predicts the loss change of query $j$ by applying the fixed target gradient $\nabla_\theta f_j(\theta_0)$ to its estimated parameter change, as in~\cref{eq:standard-group-influence}.
The reference is $f_j(\theta_{\mathcal S}^{\rm PBRF})-f_j(\theta_0)$.

\subsection{Datasets and original models}
\label[appendix]{app:datasets}

\begin{table}[htbp]
\centering\small
\caption{Models and training-set sizes. The ridge is shared by $G$ and the PBRF proximal term; it is distinct from training regularization.}
\label{tab:appendix-settings}
\begin{tabular}{@{}lp{.43\linewidth}rrr@{}}
\toprule
Dataset & Model & Classes & $N$ & Ridge\\
\midrule
Dogfish & Logistic head, 2,048 features + intercept & 2 & 1,800 & 0.01\\
MNIST & ReLU MLP, $784$--$512$--$256$--$10$ & 10 & 60,000 & 0.1001\\
CIFAR-10 & Myrtle/DAWNBench ResNet9 & 10 & 50,000 & 0.3005\\
\bottomrule
\end{tabular}
\end{table}

\paragraph{Dogfish.}
The head uses fixed Inception-v3 features~\citep{koh2017understanding,koh2019accuracy}, an $\ell_2$-regularized logistic initialization, and full-batch L-BFGS refinement.
Its objective is mean binary cross-entropy plus $10^{-4}\|\theta\|^2$, including the intercept.

\paragraph{MNIST.}
We use Adam for 30 epochs (learning rate $10^{-3}$, batch size 256), then full-batch L-BFGS (learning rate $0.5$, history 20, at most 80 steps).
The objective is mean cross-entropy plus $\tfrac12 10^{-4}\|\theta\|^2$.

\paragraph{ResNet9.}
We use Nesterov SGD for 24 epochs (batch size 512, momentum $0.9$, weight decay $5\times10^{-4}$).
The learning rate follows a piecewise-linear $0\to0.4\to0$ schedule, peaking at epoch 5.
Training uses channel normalization, random crops with padding 4, horizontal flips, and $8\times8$ Cutout.
Full-batch L-BFGS refinement uses clean inputs, history 10, and successive (learning rate, step budget) pairs $(0.1,10)$, $(0.01,20)$, and $(0.01,60)$, without line search.
Influence and PBRF evaluations use clean inputs; ResNet9 batch-normalization buffers remain fixed.

\subsection{Groups, targets, and evaluation}
\label[appendix]{app:diagnostic-details}

\paragraph{Sampling and targets.}
Each group is sampled uniformly without replacement, without class quotas; independently drawn groups may overlap.
Thus, $100$ groups need not be disjoint subsets of the training set.
Dogfish queries come from its 400-example evaluation pool.
MNIST uses ten queries per class from a 6,000-example evaluation pool, disjoint from the 4,000 validation examples; ResNet9 also uses ten test queries per class.
The class profiles in~\cref{app:decomposition-details} instead average over each class in the full 10,000-example MNIST test set.

\paragraph{Query-wise metrics.}
For each fixed query, we compare the method's predictions with the actual PBRF loss changes across the $100$ groups, using
\begin{equation*}
 R^2=1-\frac{\mathrm{SSE}}{\mathrm{SST}}.
\end{equation*}
Here SSE is the sum of squared prediction errors, and SST is the sum of squared deviations of the reference values from their mean across groups.
No regression, affine recalibration, or rescaling is fitted to the predictions.
Spearman correlation measures ranking across groups for the same query, whereas raw $R^2$ measures scale-normalized value accuracy and may be negative.
We report medians and, where indicated, interquartile ranges across queries.
Measured results are displayed to three significant figures; a displayed value of $1.00$ may reflect rounding rather than exact agreement.
Tests and reported differences use unrounded scores, and reproducibility-critical settings retain their exact values.
We neither average query losses before computing these metrics nor pool query--group pairs into one correlation.

Within each comparison, all methods and both metrics use a common query mask.
A query is excluded if its reference or any compared prediction is numerically constant: the range across groups is at most $10^{-14}$ plus $10^{-10}$ times the largest absolute value.
This also excludes constant predictions for which raw $R^2$ alone would be defined.
The main Dogfish and MNIST comparisons retain all $100$ queries; the ResNet9 four-method comparison retains $97$.
All retain $100$ groups.

\paragraph{BI normalization and figure selection.}
The BI plots use $m\mathcal B_{\mathrm{proj}}(\mathcal S)$ without the common removal factor $N^{-2}$.
Full BI and Euclidean dispersion use the corresponding summed normalization, with Euclidean dispersion defined by $\tfrac12\sum_{i\in\mathcal S}\|z_i-\bar z_{\mathcal S}\|^2$.
These fixed-size positive scalings do not change Spearman correlation with squared GI error.
Each BI panel uses the query nearest the median query-wise BI--error correlation for the corresponding model.
Prediction-scatter annotations summarize median query-wise performance, not the scores for the displayed query.
\Cref{app:geometry-ablation} compares all three dispersion measures on Dogfish and ten-class MNIST at $1\%$ removal.
The three-model projected-BI size sweep is reported separately in~\cref{app:group-size-ablation}.

\paragraph{Paired significance tests.}
\label[appendix]{app:statistical-testing}
For each metric, we pair the Euclidean Joint or Intra E2 score at $k=2$ with the single-component E2 ($k=1$) score for the same query.
We use two-sided Wilcoxon signed-rank tests~\citep{wilcoxon1945individual} of the query-wise differences.
The null is that the paired-difference distribution is symmetric about zero; this is not a test solely of the difference between the two marginal medians.
We round differences to $12$ decimal places before ranking to merge floating-point ties, discard zero differences from the rank statistic, assign average ranks to tied absolute differences, and use the tie-corrected normal approximation without continuity correction (SciPy 1.13.1).
The displayed medians use the original unrounded scores and the common query mask above.

Holm step-down correction~\citep{holm1979simple} is applied to one family of $26$ tests: $13$ available method--setting comparisons times two metrics, raw $R^2$ and Spearman.
The comparisons are Joint at Dogfish $1\%$ and Joint/Intra at Dogfish $10\%$/$30\%$, MNIST $1\%$/$10\%$/$30\%$, and ResNet9 $1\%$.
All comparisons, including decreases, enter this family.
\Cref{tab:appendix-significance} gives the medians and adjusted $p$-values.
We describe a significant difference as an improvement only if both the difference of marginal medians and the median paired difference are positive.
The Spearman tests compare the two methods' query-wise correlations; they do not test correlation against zero.
These query-paired tests condition on the fixed checkpoint and sampled removal groups.

\subsection{Curvature and component construction}
\label[appendix]{app:clustering-implementation}

We use the full-training-set GGN with the ridge in~\cref{tab:appendix-settings} and removal weight $\epsilon=-1/N$.
Per-example gradients and edit curvature exclude training regularization, which is not added again to the PBRF objective.
Dogfish uses dense FP64 algebra; MNIST uses matrix-free FP64 products, with checked relative residuals of $10^{-8}$ for GI and at most $5\times10^{-8}$ for component directions.
ResNet9 uses mixed-precision inverse solves with FP64 residual checks at relative tolerance $1.1\times10^{-6}$.

Bregman and Euclidean clustering apply the $G^{-1}$ and identity metrics to raw gradients, respectively.
Euclidean clustering does not replace $G$ in the response directions or the common E2 term.
Dogfish uses ten Lloyd restarts~\citep{lloyd1982least}, a random first center followed by farthest-point initialization, and at most 300 iterations per restart; it stops on unchanged assignments or relative center movement at most $10^{-8}$ and selects the lowest-SSE restart.
MNIST and ResNet9 use two restarts with balanced random initial assignments, at most 100 iterations each, and convergence on unchanged assignments.
We require both restarts to converge and retain the partition with the lower clustering objective.

\subsection{PBRF reference computation}
\label[appendix]{app:pbrf-reference}

Each group's actual PBRF reference point is computed once using~\cref{eq:pbrf} and reused across methods, component counts, and queries.
\Cref{tab:reference-settings} specifies the initialization and optimization budget.
The MNIST budget consists of up to 200 optimizer iterations, not necessarily 200 executed iterations: individual calls may terminate early.
The reported comparisons use all $100$ group references per setting; the query mask above is not a group-level solver exclusion.

\begin{table}[htbp]
\centering\small
\setlength{\tabcolsep}{4pt}
\caption{Solver settings used to compute the actual PBRF reference points.}
\label{tab:reference-settings}
\begin{tabular}{@{}lp{.18\linewidth}p{.34\linewidth}p{.30\linewidth}@{}}
\toprule
Model & Precision / init. & Solver & Budget and stopping \\
\midrule
Logistic
& FP64; GI initialization
& Newton--PCG with strong-Wolfe L-BFGS fallback
& PCG up to 100 steps; L-BFGS fallback up to 2{,}000 steps;
  gradient tolerance $10^{-10}$ \\

MLP
& FP64; GI initialization
& Strong-Wolfe L-BFGS, learning rate 1, history 50
& 20 calls, each with up to 10 iterations \\

ResNet
& Original-model initialization
& Strong-Wolfe L-BFGS, learning rate 0.5, history 20
& 100 FP32 warm steps followed by 40 FP64 steps \\
\bottomrule
\end{tabular}
\end{table}

\subsection{E2 implementation}
\label[appendix]{app:e2-settings}

We use the component directions $\Delta\theta_r$ and coefficients $\alpha_r$ from~\cref{eq:cluster-optimization-space}, without changing the partition or recomputing directions during optimization.
The following equations hold for each $r=1,\ldots,k$.
For removal ($\epsilon=-1/N$), differentiating the Joint objective in~\cref{eq:cluster-e2-objective} gives the $k\times k$ linear system
\begin{equation}
 \sum_{s=1}^k\alpha_s^\star\,
 \Delta\theta_r^\top
 \left(G-\frac1N\sum_{i\in\mathcal S}J_i^\top H_iJ_i\right)
 \Delta\theta_s
 =\frac1N\Delta\theta_r^\top c_{\mathcal S}.
 \label{eq:appendix-joint-system}
\end{equation}
For Intra, differentiating~\cref{eq:intra-e2-objective} instead gives
\begin{equation}
 \sum_{s=1}^k\alpha_s^\star\,\Delta\theta_r^\top G\Delta\theta_s
 -\frac{\alpha_r^\star}{N}\sum_{i\in\mathcal C_r}
      \Delta\theta_r^\top J_i^\top H_iJ_i\Delta\theta_r
 =\frac1N\Delta\theta_r^\top c_{\mathcal S}.
 \label{eq:appendix-intra-system}
\end{equation}
We solve these systems directly after checking positive definiteness and verify the solve residuals, without method-specific damping.
Coefficients are unconstrained: no simplex, nonnegativity, or target-dependent fitting is used.

For Joint E2, the parameter-space curvature is the retained-example GGN plus the positive ridge, so the coefficient matrix is positive definite when the component directions are linearly independent; the corresponding property is checked separately for Intra.
Intra removes only parts of the edit curvature: the common terms $\Delta\theta_r^\top G\Delta\theta_s$ generally remain nonzero between different components, and the coefficients are still jointly optimized.
At $k=1$, both partition geometries give the same single component; Joint and Intra coincide and only rescale GI.
Minimizing either surrogate does not guarantee an improvement against actual PBRF for every group or query.

\paragraph{Full E2.}
Full E2 minimizes the same Joint quadratic objective over the full parameter space, without restricting the parameter change to the component span.
Dogfish uses FP64 PCG with a dense Cholesky fallback.
MNIST uses FP64 matrix-free CG with an independently checked relative residual of $10^{-8}$.
Its linear query readout is evaluated against the same actual PBRF reference.
Full E2 is an approximation baseline, not the nonlinear reference or a guaranteed upper bound on prediction accuracy.
\Cref{tab:appendix-correction-main} reports the available full-space results.

\subsection{Component-analysis protocol}
\label[appendix]{app:decomposition-details}

We use the original full-group Bregman $k=2$ splits of the MNIST $10\%$ groups and order components by size.
Contributions are totals over training examples, not per-example averages.
For each component, we form a ten-dimensional class profile by averaging its linear influence $\nabla_\theta f_j(\theta_0)^\top\Delta\theta_r$ within each test class, using the full 10,000-example test set.
Cosine similarity compares the two ten-dimensional profiles within each group; the illustrated profiles are unit-normalized.
These class averages are distinct from the $100$ individual-query targets used to score prediction accuracy.

Omission retains the full-group PBRF reference and sets one component contribution to zero.
GI assigns unit weight to both components; single-component E2 ($k=1$) applies the same optimized weight to both; $k=2$ Joint E2 uses the two optimized component weights.
Large-only and small-only predictions retain the corresponding contribution at its full-fit weight and set the other to zero.
We neither refit the remaining coefficient nor change the removed set to the retained component.
\Cref{app:omission-results} reports the class-pattern case and the full/large-only/small-only prediction results.

\subsection{Computational complexity and timing protocol}
\label[appendix]{app:e2-complexity}

Let $m$ be group size, $d$ parameter dimension, $k$ component count, $t$ a bound on the total Lloyd iterations across restarts, $\mathcal C_{\rm inv}$ the cost of an inverse-$G$ solve, and $\gamma$ the per-example cost of a directional GGN product.
We assume additive work over examples when counting GGN products.
We separate partition construction, component-response solves, edit curvature, and the coefficient solve.

\paragraph{Partition construction.}
Given explicit gradients, a center-contrast implementation has costs
\begin{align}
 T_{\rm Bregman,partition}
 &=O\!\left(t\bigl[(k-1)\mathcal C_{\rm inv}+mkd\bigr]\right),\nonumber\\
 T_{\rm Euclidean,partition}&=O(tmkd).
 \label{eq:e2-clustering-complexity}
\end{align}
These expressions exclude gradient construction and shared model/data preparation.
Euclidean clustering removes the $k-1$ center-contrast metric solves per Lloyd iteration, not the inverse solves needed to form the final component responses.
The matrix-free implementation evaluates the same geometry through products rather than forming a $d\times d$ inverse.
For dense Dogfish, a Cholesky factorization of $G$ costs $O(d^3)$ once per checkpoint, whitening $m$ gradients costs $O(md^2)$, and Lloyd work costs $O(tmkd)$.
Euclidean clustering omits the metric preparation/solves, but the original $G$ is still used after partitioning.

\paragraph{Correction at fixed directions.}
Joint evaluates each of the $k$ directions over all $m$ removed examples, whereas Intra evaluates an example only along its own component's direction.
With vector-valued products and coefficient-space projection, the edit costs are
\begin{equation}
 T_{\rm Joint,edit}=O(km\gamma+dk^2),
 \qquad
 T_{\rm Intra,edit}=O(m\gamma+dk).
 \label{eq:e2-edit-complexity}
\end{equation}
The example--direction count drops from $km$ to $\sum_r|\mathcal C_r|=m$, regardless of component balance.
Both methods retain the global component-response solves, with cost $O(k\mathcal C_{\rm inv})$, the common matrix with entries $\Delta\theta_r^\top G\Delta\theta_s$, and a dense $k\times k$ coefficient solve costing $O(k^3)$.
Forming the common matrix from scratch requires $k$ applications of $G$ and $O(dk^2)$ inner-product work; cached identities may reduce this work.
Intra does not turn $G$ into smaller parameter-space blocks.
Its edit cost is of the same order as single-component E2 ($k=1$), so these expressions alone do not imply an end-to-end speedup over $k=1$.

\paragraph{Partition timing.}
We use $100$ MNIST $1\%$ groups at $k=2$, FP64, with each method pair on the same GPU: 82 A100-SXM4-80GB, five A100-SXM4-40GB, and 13 H200 NVL pairs.
Each method is timed once per group; Bregman precedes Euclidean, without dedicated timing warm-up.
Both use two restarts and unchanged-assignment stopping, with a 100-iteration cap per restart.
Both restarts converge for every group.
The total iterations across restarts have a median of nine for each method, with ranges of 6--12 for Bregman and 6--16 for Euclidean.
The wall-clock interval includes group-gradient preparation, metric solves within clustering, Lloyd assignments, restart selection, and progress serialization.
Shared model/data/operator construction, baseline validation, final component-response solves, E2, and query readout are excluded.
Shared preparation is neither timed nor amortized into these values.
Scalar and CPU operations synchronize within the computation; no extra CUDA synchronization was placed at the timer boundaries.
These are paired stage measurements over different groups, not repeated cold-start benchmarks.

\paragraph{Edit-curvature timing.}
\Cref{tab:e2-cost} uses the $49$ matched MNIST $10\%$ groups executed on an RTX3090 in FP64, with fixed Euclidean partitions and cached directions.
Each method is timed once per group in the existing worker order, rather than in independently warmed-up repeated trials.
The Joint interval excludes coefficient-space projection; the Intra interval includes scalar direction products and auxiliary vector accumulation.
Both exclude partition construction, response inverse solves, the common full-source Gram computation, the coefficient solve, and final query readout.
The intervals therefore document the instrumented edit stages, not identical low-level instruction sets or end-to-end runtimes.
\Cref{app:runtime-results} reports the measured partition and edit-curvature costs under these protocols.

\newpage
\section{Supplementary Experimental Results}
\label[appendix]{app:supplementary-results}

We distinguish the error diagnostic from prediction accuracy, and partition geometry from the curvature retained in the response and correction.
All measured results use the protocol in~\cref{app:experimental-details}; table entries are query-wise medians unless stated otherwise.

\subsection{Group-size ablation}
\label[appendix]{app:group-size-ablation}

\paragraph{BI and squared GI error.}
\Cref{tab:appendix-bi-size} varies the removal ratio over $1,2,5,10,20,35,50\%$ for all three models.
Each ratio uses $100$ random groups and the same $100$ queries within a model's sweep.
For every query, we correlate projected BI with squared GI error across groups, then summarize the $100$ correlations by their median.
The association weakens with removal size: between $1\%$ and $50\%$, the median decreases from $0.865$ to $0.140$ on Dogfish, from $0.701$ to $0.121$ on MNIST, and from $0.525$ to $0.145$ on ResNet9.
These diagnostic correlations are not the prediction--reference correlations used to evaluate correction methods.

\begin{table}[htbp]
\centering\small
\setlength{\tabcolsep}{4pt}
\caption{Removal-ratio ablation for projected BI versus squared GI error. Each entry is the median of $100$ query-wise Spearman correlations, each computed across $100$ random groups. The original model and query bank are fixed within each model's size sweep.}
\label{tab:appendix-bi-size}
\begin{tabular}{@{}lrrrrrrr@{}}
\toprule
Model & $1\%$ & $2\%$ & $5\%$ & $10\%$ & $20\%$ & $35\%$ & $50\%$ \\
\midrule
Dogfish logistic & 0.865 & 0.815 & 0.676 & 0.549 & 0.475 & 0.327 & 0.140 \\
MNIST-10 MLP & 0.701 & 0.656 & 0.580 & 0.453 & 0.340 & 0.169 & 0.121 \\
CIFAR-10 ResNet9 & 0.525 & 0.413 & 0.290 & 0.218 & 0.190 & 0.146 & 0.145 \\
\bottomrule
\end{tabular}
\end{table}

\paragraph{Correction accuracy and paired comparisons.}
\Cref{tab:appendix-significance} compares single-component E2 ($k=1$) with Euclidean Joint and Intra E2 ($k=2$) for the available removal ratios.
Dogfish and MNIST cover $1,10,30\%$; ResNet9's correction comparison is at $1\%$.
The Dogfish $1\%$ Intra comparison is unavailable and is not filled by a result from another setting.
The group and query banks are matched across methods within each setting; the diagnostic and correction sweeps are separate experiments.
The paired tests follow~\cref{app:diagnostic-details} and use all common queries, not a selected representative query.

On MNIST at $10\%$, both Joint and Intra improve raw $R^2$ significantly in this query-paired analysis, although the Intra median gain is only $0.000635$.
Statistical significance therefore should not be confused with the magnitude of the improvement.
At MNIST $30\%$, Intra instead decreases raw $R^2$ significantly while improving Spearman, separating value accuracy from ranking accuracy.

\begin{table}[htbp]
\centering\small
\setlength{\tabcolsep}{4pt}
\caption{Median query-wise accuracy and two-sided query-paired $p$-values for Euclidean E2 ($k=2$) versus single-component E2 ($k=1$).
Holm correction covers all $26$ comparisons jointly across settings, variants,
and both metrics (\cref{app:diagnostic-details}).
The Spearman tests compare methods and are not tests of zero correlation.
MNIST $30\%$ Intra shows a significant decrease in raw $R^2$, not an improvement.}
\label{tab:appendix-significance}
\begin{tabular}{@{}llrrrrr@{}}
\toprule
Removal & Metric & $k=1$ & Joint $k=2$ & $p_{\rm adj}$ & Intra $k=2$ & $p_{\rm adj}$ \\
\midrule

\multicolumn{7}{@{}l}{\textit{Dogfish}} \\

$1\%$
& $R^2$      & 0.995 & 0.999 & $1.01\times10^{-16}$ & \multicolumn{2}{c}{---} \\
& $\rho$    & 0.998 & 1.00  & $1.35\times10^{-16}$ & \multicolumn{2}{c}{---} \\ \cmidrule{2-7}
$10\%$
& $R^2$      & 0.946 & 0.985 & $1.01\times10^{-16}$ & 0.982 & $1.01\times10^{-16}$ \\
& $\rho$    & 0.985 & 0.994 & $1.01\times10^{-16}$ & 0.993 & $1.15\times10^{-16}$ \\ \cmidrule{2-7}
$30\%$
& $R^2$      & 0.891 & 0.945 & $1.40\times10^{-16}$ & 0.945 & $1.65\times10^{-15}$ \\
& $\rho$    & 0.981 & 0.989 & $1.36\times10^{-15}$ & 0.984 & $2.87\times10^{-5}$ \\

\midrule
\multicolumn{7}{@{}l}{\textit{MNIST}} \\
$1\%$
& $R^2$      & 0.998 & 1.00  & $2.83\times10^{-16}$ & 1.00  & $3.38\times10^{-16}$ \\
& $\rho$    & 0.999 & 1.00  & $1.02\times10^{-16}$ & 1.00  & $1.02\times10^{-16}$ \\ \cmidrule{2-7}
$10\%$
& $R^2$      & 0.993 & 0.995 & $5.64\times10^{-15}$ & 0.994 & $2.22\times10^{-6}$ \\
& $\rho$    & 0.997 & 0.998 & $3.16\times10^{-15}$ & 0.998 & $1.24\times10^{-14}$ \\ \cmidrule{2-7}
$30\%$
& $R^2$      & 0.964 & 0.968 & $4.23\times10^{-9}$  & 0.962 & $2.34\times10^{-4}$ \\
& $\rho$    & 0.992 & 0.993 & $3.46\times10^{-7}$  & 0.992 & $1.96\times10^{-4}$ \\

\midrule
\multicolumn{7}{@{}l}{\textit{ResNet9}} \\
$1\%$
& $R^2$      & 0.965 & 0.975 & $2.83\times10^{-16}$ & 0.974 & $2.83\times10^{-16}$ \\
& $\rho$    & 0.986 & 0.989 & $7.27\times10^{-16}$ & 0.989 & $5.37\times10^{-16}$ \\

\bottomrule
\end{tabular}
\end{table}

\Cref{tab:appendix-size-sensitivity} includes GI and additional $k=10$ comparisons at these supplementary sizes.
For MNIST at $30\%$, Euclidean Intra's median $R^2$ is below single-component E2 ($k=1$) at both $k=2$ and $k=10$, while Joint improves as more directions are available.
This is an observed accuracy--cost trade-off, not a guarantee that either approximation is uniformly preferable.

\begin{table}[htbp]
\centering\small
\setlength{\tabcolsep}{5pt}
\caption{Additional removal ratios, with $100$ random groups and $100$ retained fixed queries per setting.
All methods within a setting share the checkpoint, groups, queries, and actual PBRF references.
The selected $k=2$ and $k=10$ comparisons expose removal-size and component-count sensitivity.}
\label{tab:appendix-size-sensitivity}
\begin{tabular}{@{}lrr@{}}
\toprule
Method & Spearman & $R^2$ \\
\midrule

\multicolumn{3}{@{}l}{\textit{Dogfish ($10\%$)}} \\
GI & 0.988 & 0.901 \\
Single-component E2 ($k=1$) & 0.985 & 0.946 \\
Euclidean Joint E2 ($k=2$) & 0.994 & 0.985 \\
Euclidean Intra E2 ($k=2$) & 0.993 & 0.982 \\
Euclidean Joint E2 ($k=10$) & 1.00 & 0.997 \\
Euclidean Intra E2 ($k=10$) & 0.990 & 0.981 \\

\midrule
\multicolumn{3}{@{}l}{\textit{MNIST ($1\%$)}} \\
GI & 0.999 & 0.990 \\
Single-component E2 ($k=1$)  & 0.999 & 0.998 \\
Euclidean Joint E2 ($k=2$) & 1.00 & 1.00 \\
Euclidean Intra E2 ($k=2$) & 1.00 & 1.00 \\
Euclidean Joint E2 ($k=10$) & 1.00 & 1.00 \\
Euclidean Intra E2 ($k=10$) & 0.999 & 0.999 \\

\midrule
\multicolumn{3}{@{}l}{\textit{MNIST ($30\%$)}} \\
GI & 0.992 & 0.921 \\
Single-component E2 ($k=1$) & 0.992 & 0.964 \\
Euclidean Joint E2 ($k=2$) & 0.993 & 0.968 \\
Euclidean Intra E2 ($k=2$) & 0.992 & 0.962 \\
Euclidean Joint E2 ($k=10$) & 0.995 & 0.977 \\
Euclidean Intra E2 ($k=10$) & 0.990 & 0.953 \\

\bottomrule
\end{tabular}
\end{table}

\subsection{Geometry ablation and the Euclidean approximation}
\label[appendix]{app:geometry-ablation}

Three considerations motivate the Euclidean approximation:
similar associations between dispersion and GI error,
similar correction accuracy across partition geometries,
and precedents in gradient-based attribution.

\paragraph{Similar dispersion--error associations.}
\Cref{tab:geometry-ablation} compares projected BI, full BI, and raw-gradient Euclidean dispersion using the same
Dogfish and MNIST $1\%$ groups and query-specific squared GI errors.
Full BI and Euclidean dispersion yield similar median correlations:
$0.759$ versus $0.760$ on Dogfish and $0.556$ versus $0.558$ on MNIST.
Projected BI instead reaches $0.865$ and $0.701$, respectively.
Thus, replacing the curvature-weighted dispersion with its Euclidean counterpart has little effect on the observed
error association, whereas target projection provides additional query-specific information.
This comparison concerns dispersion measures; the effect of partition geometry on correction accuracy is examined
next.

\begin{table}[htbp]
\centering\small
\setlength{\tabcolsep}{4pt}
\caption{Geometry ablation on the same random $1\%$ groups and fixed queries as the BI experiment. Entries are medians over $100$ query-wise Spearman correlations with squared GI error, each computed across $100$ groups. Euclidean dispersion uses raw gradients, not parameter responses.}
\label{tab:geometry-ablation}
\begin{tabular}{@{}lrrr@{}}
\toprule
Model & Projected BI & Full BI & Euclidean \\
\midrule
Dogfish logistic & 0.865 & 0.759 & 0.760 \\
MNIST-10 MLP & 0.701 & 0.556 & 0.558 \\
\bottomrule
\end{tabular}
\end{table}

\paragraph{Similar correction accuracy.}
\Cref{tab:clustering-curvature} compares Bregman and Euclidean partitions under both Joint and Intra E2 at $k=2$ on
Dogfish $30\%$ and MNIST $10\%$ groups.
Only the partition metric changes; component responses and coefficient fitting retain the original $G$.
The two geometries yield similar prediction accuracy under either curvature scope.
For Joint E2, median raw $R^2$ rounds to $0.945$ on Dogfish and $0.995$ on MNIST for both geometries.
\Cref{tab:appendix-ksweep-dogfish30,tab:appendix-ksweep-mnist10,tab:appendix-ksweep-dogfish10} extend this
comparison across component counts.
These results support the approximation empirically rather than establishing equivalence between the partitions.

\paragraph{Precedents in gradient-based attribution.}
Gradient-inner-product attribution provides a related precedent for omitting curvature weighting.
TracIn~\citep{pruthi2020estimating} approximates training-induced loss changes using gradient inner products
accumulated along training checkpoints.
At a single checkpoint, this inner-product form is related algebraically to replacing inverse curvature by the
identity.
The full trajectory-based method is not identical to that substitution in classical influence functions.
In our method, the identity metric is used only to construct the partition; component responses and coefficient
fitting still use the original $G$.

Together, these considerations motivate Euclidean clustering as a practical partition approximation without
requiring $G$ to be close to the identity or the two geometries to produce identical assignments.
The resulting partition-construction savings are reported in~\cref{app:runtime-results}, with complexity and
measurement protocols in~\cref{app:e2-complexity}.

\subsection{Component count and Joint--Intra correction}
\label[appendix]{app:k-sweep}

\paragraph{Rescaling, reweighting, and the full-space comparison.}
\Cref{tab:appendix-correction-main} collects the main correction settings: Dogfish at $30\%$ removal, MNIST at $10\%$, and ResNet9 at $1\%$.
The $k=1$ comparison isolates rescaling along the original GI direction; $k>1$ also permits changes in relative component contributions.
At $k=1$, both geometries give the same single component, and Joint and Intra E2 coincide.
Full E2 removes the component-span restriction but retains the same quadratic edit approximation.
Its stronger results on Dogfish and MNIST indicate room for improvement beyond the restricted response space; they do not make Full E2 equivalent to actual PBRF.

\begin{table}[htbp]
\centering\small
\setlength{\tabcolsep}{5pt}
\caption{Correction results against actual PBRF reference points.
All settings use $100$ random groups; metrics are medians over the common retained queries
($100$ for Dogfish and MNIST, and $97$ for ResNet9).
All methods within a setting share the checkpoint, groups, queries, and actual PBRF references.
Full E2 is an additional approximation, not the reference itself.}
\label{tab:appendix-correction-main}
\begin{tabular}{@{}lrr@{}}
\toprule
Method & Spearman & $R^2$ \\
\midrule

\multicolumn{3}{@{}l}{\textit{Dogfish ($30\%$)}} \\
GI & 0.984 & 0.862 \\
Single-component E2 ($k=1$) & 0.981 & 0.891 \\
Bregman Joint E2 ($k=2$) & 0.989 & 0.945 \\
Bregman Intra E2 ($k=2$) & 0.984 & 0.945 \\
Euclidean Joint E2 ($k=2$) & 0.989 & 0.945 \\
Euclidean Intra E2 ($k=2$) & 0.984 & 0.945 \\
Full E2 & 1.00 & 0.990 \\

\midrule
\multicolumn{3}{@{}l}{\textit{MNIST ($10\%$)}} \\
GI & 0.997 & 0.978 \\
Single-component E2 ($k=1$) & 0.997 & 0.993 \\
Bregman Joint E2 ($k=2$) & 0.998 & 0.995 \\
Bregman Intra E2 ($k=2$) & 0.998 & 0.993 \\
Euclidean Joint E2 ($k=2$) & 0.998 & 0.995 \\
Euclidean Intra E2 ($k=2$) & 0.998 & 0.994 \\
Full E2 & 0.999 & 0.999 \\

\midrule
\multicolumn{3}{@{}l}{\textit{ResNet9 ($1\%$)}} \\
GI & 0.986 & 0.891 \\
Single-component E2 ($k=1$) & 0.986 & 0.965 \\
Euclidean Joint E2 ($k=2$) & 0.989 & 0.975 \\
Euclidean Intra E2 ($k=2$) & 0.989 & 0.974 \\

\bottomrule
\end{tabular}
\end{table}

\paragraph{Complete component-count sweeps.}
\Cref{tab:appendix-ksweep-dogfish30,tab:appendix-ksweep-mnist10} report $k\in\{1,2,4,6,8,10\}$ on the two main correction cohorts.
\Cref{tab:appendix-ksweep-dogfish10} additionally reports Dogfish at $10\%$ removal, allowing comparison with MNIST at the same removal ratio.
For Dogfish $10\%$, median raw $R^2$ increases from $0.901$ for GI to $0.946$ for $k=1$ E2 and $0.985$ for Bregman Joint E2 at $k=2$; Euclidean Intra E2 at $k=2$ reaches $0.982$.
For MNIST $10\%$, the corresponding Bregman Joint improvement is smaller, from $0.993$ at $k=1$ to $0.995$ at $k=2$.
Bregman and Euclidean denote only the partition metric; both retain the full $G$ in component responses and coefficient fitting.
Joint and Intra differ only in the additional edit-curvature term, as specified in~\cref{eq:appendix-joint-system,eq:appendix-intra-system}.

Euclidean and Bregman Joint E2 have similar accuracy in these settings.
Intra saves edit-curvature work but need not preserve Joint's accuracy: its raw $R^2$ decreases at larger $k$ in the reported cohorts.
The results therefore do not establish that inter-component edit curvature is universally negligible or that increasing $k$ always improves the approximation.
Partitions obtained at different $k$ need not define nested response spaces.

\begin{table}[htbp]
\centering\small
\setlength{\tabcolsep}{3.5pt}
\caption{Dogfish, $30\%$: component-count sweep on the same $100$ groups and
$100$ queries. Each entry reports the median over fixed test queries. The four variants
coincide at $k=1$.}
\label{tab:appendix-ksweep-dogfish30}
\begin{tabular}{@{}rcccccccc@{}}
\toprule
& \multicolumn{4}{c}{Spearman $\rho$} & \multicolumn{4}{c}{$R^2$} \\
\cmidrule(lr){2-5}\cmidrule(lr){6-9}
& \multicolumn{2}{c}{Joint E2} & \multicolumn{2}{c}{Intra E2}
& \multicolumn{2}{c}{Joint E2} & \multicolumn{2}{c}{Intra E2} \\
\cmidrule(lr){2-3}\cmidrule(lr){4-5}
\cmidrule(lr){6-7}\cmidrule(lr){8-9}
$k$
& Bregman & Euclidean & Bregman & Euclidean
& Bregman & Euclidean & Bregman & Euclidean \\
\midrule
1  & 0.981 & 0.981 & 0.981 & 0.981 & 0.891 & 0.891 & 0.891 & 0.891 \\
2  & 0.989 & 0.989 & 0.984 & 0.984 & 0.945 & 0.945 & 0.945 & 0.945 \\
4  & 0.993 & 0.993 & 0.984 & 0.984 & 0.968 & 0.970 & 0.951 & 0.950 \\
6  & 0.995 & 0.996 & 0.982 & 0.982 & 0.978 & 0.978 & 0.946 & 0.947 \\
8  & 0.996 & 0.996 & 0.980 & 0.979 & 0.981 & 0.981 & 0.941 & 0.941 \\
10 & 0.997 & 0.997 & 0.977 & 0.978 & 0.983 & 0.983 & 0.939 & 0.940 \\
\bottomrule
\end{tabular}
\end{table}

\begin{table}[htbp]
\centering\small
\setlength{\tabcolsep}{3.5pt}
\caption{MNIST, $10\%$: component-count sweep on the same $100$ groups and
$100$ queries. Each entry reports the median over fixed test queries. The four variants
coincide at $k=1$.}
\label{tab:appendix-ksweep-mnist10}
\begin{tabular}{@{}rcccccccc@{}}
\toprule
& \multicolumn{4}{c}{Spearman $\rho$} & \multicolumn{4}{c}{$R^2$} \\
\cmidrule(lr){2-5}\cmidrule(lr){6-9}
& \multicolumn{2}{c}{Joint E2} & \multicolumn{2}{c}{Intra E2}
& \multicolumn{2}{c}{Joint E2} & \multicolumn{2}{c}{Intra E2} \\
\cmidrule(lr){2-3}\cmidrule(lr){4-5}
\cmidrule(lr){6-7}\cmidrule(lr){8-9}
$k$
& Bregman & Euclidean & Bregman & Euclidean
& Bregman & Euclidean & Bregman & Euclidean \\
\midrule
1  & 0.997 & 0.997 & 0.997 & 0.997 & 0.993 & 0.993 & 0.993 & 0.993 \\
2  & 0.998 & 0.998 & 0.998 & 0.998 & 0.995 & 0.995 & 0.993 & 0.994 \\
4  & 0.998 & 0.998 & 0.998 & 0.998 & 0.996 & 0.996 & 0.994 & 0.994 \\
6  & 0.999 & 0.999 & 0.998 & 0.998 & 0.996 & 0.997 & 0.994 & 0.994 \\
8  & 0.999 & 0.999 & 0.998 & 0.998 & 0.997 & 0.997 & 0.993 & 0.994 \\
10 & 0.999 & 0.999 & 0.998 & 0.998 & 0.997 & 0.997 & 0.993 & 0.993 \\
\bottomrule
\end{tabular}
\end{table}

\begin{table}[htbp]
\centering\small
\setlength{\tabcolsep}{3.5pt}
\caption{Dogfish, $10\%$: component-count sweep on the same $100$ groups and
$100$ queries. Each entry reports the median over fixed test queries. The four variants
coincide at $k=1$.}
\label{tab:appendix-ksweep-dogfish10}
\begin{tabular}{@{}rcccccccc@{}}
\toprule
& \multicolumn{4}{c}{Spearman $\rho$} & \multicolumn{4}{c}{$R^2$} \\
\cmidrule(lr){2-5}\cmidrule(lr){6-9}
& \multicolumn{2}{c}{Joint E2} & \multicolumn{2}{c}{Intra E2}
& \multicolumn{2}{c}{Joint E2} & \multicolumn{2}{c}{Intra E2} \\
\cmidrule(lr){2-3}\cmidrule(lr){4-5}
\cmidrule(lr){6-7}\cmidrule(lr){8-9}
$k$
& Bregman & Euclidean & Bregman & Euclidean
& Bregman & Euclidean & Bregman & Euclidean \\
\midrule
1  & 0.985 & 0.985 & 0.985 & 0.985 & 0.946 & 0.946 & 0.946 & 0.946 \\
2  & 0.994 & 0.994 & 0.993 & 0.993 & 0.985 & 0.985 & 0.982 & 0.982 \\
4  & 0.998 & 0.998 & 0.994 & 0.994 & 0.995 & 0.996 & 0.986 & 0.987 \\
6  & 0.999 & 0.999 & 0.993 & 0.994 & 0.997 & 0.997 & 0.984 & 0.985 \\
8  & 0.999 & 0.999 & 0.992 & 0.992 & 0.997 & 0.997 & 0.983 & 0.983 \\
10 & 1.00  & 1.00  & 0.991 & 0.990 & 0.997 & 0.997 & 0.981 & 0.981 \\
\bottomrule
\end{tabular}
\end{table}

\subsection{Why retain and reweight both components?}
\label[appendix]{app:omission-results}

\paragraph{Distinct class-wise patterns.}
Using the profiles defined in~\cref{app:decomposition-details}, the displayed MNIST $10\%$ group has component sizes $5{,}997$ and $3$ and cosine similarity $-0.457$.
Across all $100$ groups, the median profile cosine is $-0.027$, the interquartile range is $[-0.250,0.165]$, and the 90th percentile is $0.359$.
Component size alone therefore does not describe its class-wise influence pattern.

\paragraph{Omission against the full-group reference.}
\Cref{tab:component-omission} retains both components, only the larger component, or only the smaller component, using the same full-group actual PBRF reference and fixed full-fit coefficients.
For Bregman Joint E2 ($k=2$), median raw $R^2$ falls from $0.995$ with both components to $0.636$ with large-only and $-0.009$ with small-only.
The full-versus-large-only comparison demonstrates the predictive contribution of the small component; full versus small-only demonstrates that of the large component.
The poor small-only prediction is not evidence that the small component can be discarded.
The same qualitative degradation occurs for GI and single-component E2 ($k=1$).
These results motivate retaining and reweighting both components rather than dropping one based on size.
They do not test a refitted one-component model or establish independence between components.

\begin{table}[htbp]
\centering\small
\setlength{\tabcolsep}{4pt}
\caption{MNIST $10\%$, Bregman Joint E2 ($k=2$): component omission against the same full-group PBRF reference using $100$ groups and $100$ queries. E2 coefficients are held at the full-fit values; they are not refitted after omission.}
\label{tab:component-omission}
\begin{tabular}{@{}llrr@{}}
\toprule
Coefficient model & Retained contributions & Spearman & $R^2$ \\
\midrule
\multirow{3}{*}{GI}
& Full       & 0.997 & 0.978 \\
& Large only & 0.794 & 0.630 \\
& Small only & 0.492 & -0.008 \\
\midrule
\multirow{3}{*}{Single-component E2 ($k=1$)}
& Full       & 0.997 & 0.993 \\
& Large only & 0.797 & 0.638 \\
& Small only & 0.492 & -0.002 \\
\midrule
\multirow{3}{*}{ Bregman Joint E2 ($k=2$)}
& Full       & 0.998 & 0.995 \\
& Large only & 0.797 & 0.636 \\
& Small only & 0.493 & -0.009 \\
\bottomrule
\end{tabular}
\end{table}

\clearpage
\subsection{Stage-wise computational savings}
\label[appendix]{app:runtime-results}

The costs in~\cref{eq:e2-clustering-complexity,eq:e2-edit-complexity} concern different stages; the measurement boundaries and hardware are specified in~\cref{app:e2-complexity}.
For MNIST $1\%$ at $k=2$, median partition-construction time decreases from $59.4$ seconds for Bregman to $0.277$ seconds for Euclidean clustering, a $99.5\%$ reduction from the two medians over $100$ matched groups.
This comparison excludes shared preparation and final component-response solves.

\Cref{tab:e2-cost} reports the instrumented edit-curvature intervals for the $49$ matched MNIST $10\%$ groups with fixed Euclidean directions on an RTX3090.
At $k=2$, Joint and Intra have medians of $0.193$ and $0.103$ seconds, respectively, a $46.6\%$ reduction.
The common inverse solves, full-source Gram computation, and coefficient solve are excluded from these intervals.
The measurements support lower cost for the replaced stages at fixed $k$, not an end-to-end speedup over single-component E2 ($k=1$).

\begin{table}[htbp]
\centering\small
\setlength{\tabcolsep}{4pt}
\caption{Median edit-curvature stage time (seconds), MNIST $10\%$, fixed Euclidean components, $49$ matched groups on an RTX3090. Reductions use the two medians. These are instrumented stage intervals, not end-to-end timings.}
\label{tab:e2-cost}
\begin{tabular}{@{}rrrr@{}}
\toprule
$k$ & Joint & Intra & Reduction (\%) \\
\midrule
2 & 0.193 & 0.103 & 46.6 \\
4 & 0.384 & 0.111 & 71.2 \\
10 & 0.973 & 0.136 & 86.0 \\
\bottomrule
\end{tabular}
\end{table}

\end{document}